\documentclass[a4paper, 12pt, oneside]{article}
\usepackage{fancyhdr}
\usepackage{setspace}
\usepackage[top=3.5cm, bottom=3cm, left=2cm, right=2cm]{geometry}
\usepackage[utf8x]{inputenc}
\usepackage{amsfonts}
\usepackage{amssymb}
\usepackage{enumitem}
\usepackage{mathtools}
\usepackage{graphicx}
\usepackage{array}
\usepackage{tikz}
\tikzset{attack/.style={-latex}} %
\tikzset{symmetric attack/.style={latex-latex}} %
\tikzset{self-attack/.style={}} %
\tikzset{pattack/.style={->>, thick}}
\tikzset{symmetric pattack/.style={<<->>, thick}}
\tikzset{self-pattack/.style={thick}}
\tikzset{normal attack/.style={-latex, thick}}
\tikzset{symmetric normal attack/.style={latex-latex, thick}}
\tikzset{normal self-attack/.style={thick}}
\tikzset{reverse attack/.style={-latex, dotted, thick}}
\tikzset{symmetric reverse attack/.style={latex-latex, dotted, thick}}
\tikzset{reverse self-attack/.style={dotted, thick}}
\tikzset{grayed/.style={fill=gray, fill opacity=0.3}}

\usepackage{amsthm}
\theoremstyle{plain}
\newtheorem{theorem}{Theorem}
\newtheorem{lemma}[theorem]{Lemma}
\newtheorem{proposition}[theorem]{Proposition}
\newtheorem{corollary}[theorem]{Corollary}
\theoremstyle{definition}
\newtheorem{definition}{Definition}
\theoremstyle{definition}
\newtheorem{principle}{Principle}
\theoremstyle{definition}
\newtheorem{axiom}{Axiom}
\theoremstyle{definition}
\newtheorem{example}{Example}

\usepackage{float}

\def\elitist{\texttt{Eli}}
\def\delitist{\texttt{DEli}}
\def\democratic{\texttt{Dem}}
\def\tleq{\ensuremath{\trianglelefteq}} % triangle less than or equal

\def\argeli{\ensuremath{\prec_{\elitist}}}

\def\argdeli{\ensuremath{\prec_{\delitist}}}

\def\argdemq{\ensuremath{\preccurlyeq_{\democratic}}}
\def\argndemq{\ensuremath{\not\preccurlyeq_{\democratic}}}
\def\argdem{\ensuremath{\prec_{\democratic}}}
\def\argndem{\ensuremath{\nprec_{\democratic}}}

\def\AF{\ensuremath{(\Args, \attacks)}} % AF
\def\Args{\ensuremath{\textit{Args}}}
\def\attacks{\ensuremath{\leadsto}}
\def\nattacks{\ensuremath{\not\leadsto}}
\def\DAF{\ensuremath{(\Args, \defeats)}} % Defeat AF
\def\defeats{\ensuremath{\hookrightarrow}}
\def\ndefeats{\ensuremath{\not\hookrightarrow}}
\def\PAF{\ensuremath{(\Args, \attacks, \preccurlyeq)}} % PAF

\def\RPAF{\ensuremath{(\Args, \attacks, \preccurlyeq, \tleq)}} % PAF
\def\Nattacks{\ensuremath{\attacks_N}} % Dung's normal attack

\def\NAF{\ensuremath{(\Args, \Nattacks)}} % NAF

\def\contrary{\ensuremath{\bar{}\,\bar{}\,\bar{}\,}}
\def\abaf{\ensuremath{(\LL, \R, \A, \contrary)}} 
\def\abafp{\ensuremath{(\LL, \R, \A, \contrary, \leqslant)}} 
\def\abafe{\ensuremath{(\LL, \R, \A, \contrary, \emptyset)}} 
 
\def\paba{\ensuremath{(\LL, \R, \A, \contrary, \preccurlyeq)}}
\def\pattacks{\ensuremath{\attacks_<}}
\def\npattacks{\ensuremath{\nattacks_<}}
\def\cn{\textit{Cn}} % consequence operator
\def\cl{\textit{Cl}} % closure operator
\def\Def{\textit{Def}} % defence operator
\def\ot{\ensuremath{\leftarrow}} % for rules, instead of leftarrow
\newcommand{\contr}[1]{\ensuremath\overline{#1}} % contrary of
\newcommand{\newcontr}[1]{\ensuremath{{#1}^c}} % new contrary
\def\asma{\ensuremath{\alpha}} % assumption alpha
\def\asmb{\ensuremath{\beta}} 
\def\asmc{\ensuremath{\gamma}} 
\def\asmd{\ensuremath{\delta}} 
\def\asme{\ensuremath{\varepsilon}} 
\def\asms{\ensuremath{s}} % assumption s
\def\asmx{\ensuremath{x}} 
\def\asmA{\ensuremath{A}} % assumption set A
\def\asmB{\ensuremath{B}} 
 
\def\asmS{\ensuremath{S}} 
\def\asmE{\ensuremath{E}} 
\def\asmI{\ensuremath{I}} 

\newcommand{\ccontrary}{\ensuremath{\CC}}

\def\To{\ensuremath{\Rightarrow}} 
\newcommand{\ccontr}[1]{\ensuremath{\CC(#1)}}
\def\prem{\texttt{Prem}} % premises
\def\conc{\texttt{Conc}} % conclusion
\def\sub{\texttt{Sub}} % sub-arguments
\def\defrules{\texttt{DefRules}} % defeasible rules
\def\aspicf{\ensuremath{(\LL, \ccontrary, \R_s, \R_d, \leqslant_d,  n, \K_n, \K_p, \leqslant_p)}}
\def\aspicfo{\ensuremath{(\LL, \ccontrary, \R_s, \R_d, \leqslant_d, \K_p, \leqslant_p)}}

\def\argA{\ensuremath{\textsf{A}}}
\def\argB{\ensuremath{\textsf{B}}}

\def\argE{\ensuremath{\textsf{E}}}
\def\argX{\ensuremath{\textsf{X}}}

\newcommand*{\abap}{ABA$^+$}
\newcommand*{\aspicp}{ASPIC$^+$}

\newcommand*{\flatsemantics}{grounded, ideal, stable, preferred, complete}

\def\leave{\ensuremath{\textit{leave}}}
\def\stay{\ensuremath{\textit{stay}}}

\def\A{\ensuremath{\mathcal{A}}} % calligraphic A

\def\CC{\ensuremath{\mathcal{C}}} 
 
\def\E{\ensuremath{\mathcal{E}}}
\def\F{\ensuremath{\mathcal{F}}}

\def\K{\ensuremath{\mathcal{K}}} 
\def\LL{\ensuremath{\mathcal{L}}}

\def\R{\ensuremath{\mathcal{R}}}

\newcommand*{\consistency}{the Axiom of Consistency}
\newcommand*{\negation}{the Axiom of Negation}

\newcommand*{\contraposition}{the Axiom of Contraposition}
\newcommand*{\wcp}{the Axiom of Weak Contraposition}

\newcommand*{\maximal}{the Principle of Maximal Elements}
\newcommand*{\conflict}{the Principle of Conflict Preservation}
\newcommand*{\emptypref}{the Principle of Empty Preferences}
\newcommand*{\closure}{the Principle of Closure}

\newcommand*{\cconsistency}{the Principle of Classical Consistency}

\usepackage{hyperref}
\hypersetup{
unicode=true,
pdftitle={ABA+: Assumption-Based Argumentation with Preferences},
pdfauthor={Kristijonas Cyras, Francesca Toni},
pdfsubject={Structured Argumentation with Preferences},
pdfkeywords={Argumentation, Preferences},
colorlinks=true,
linkcolor=gray,
citecolor=black,
urlcolor=black,
filecolor=black,
pdfstartview=FitH,
pdfborder={0 0 0}
}

\newcommand{\citet}[1]{\cite{#1}}

\title{\abap: Assumption-Based Argumentation with Preferences}
\author{Kristijonas \v Cyras\thanks{Corresponding author: \href{mailto:k.cyras13@imperial.ac.uk}{{\tt k.cyras13}@{\tt imperial.ac.uk}}}\and Francesca Toni}
\date{Imperial College London \\ \today} 

\begin{document}

\pagestyle{fancy}
\setlist{nolistsep}

\maketitle

\begin{abstract}
We present \abap, a new approach to handling preferences in a well known structured argumentation formalism, Assumption-Based Argumentation (ABA). 
In \abap, preference information given over assumptions is incorporated directly into the attack relation, thus resulting in attack reversal. 
\abap\ conservatively extends ABA and exhibits various desirable features regarding relationship among argumentation semantics as well as preference handling. 
We also introduce Weak Contraposition, a principle concerning reasoning with rules and preferences that relaxes the standard principle of contraposition, 
while guaranteeing additional desirable features for \abap.
\end{abstract}

%                                INTRODUCTION
\section{Introduction}
\label{sec:Introduction}

Argumentation (as overviewed in e.g.~\cite{Rahwan:Simari:2009}) is a branch of Knowledge Representation and Reasoning that, among other goals, aims to formalise reasoning with conflicting and uncertain information. 
In argumentation, knowledge is often (e.g.~in \cite{Dung:1995}) represented via \emph{arguments} and conflicts are captured via \emph{attacks} among arguments, 
and reasoning amounts to selecting sets of collectively acceptable arguments, called \emph{extensions}, where acceptability conditions are dependent on the \emph{semantics} chosen (see e.g.~\cite{Baroni:Giacomin:2009,Baroni:Caminada:Giacomin:2011} for overviews of argumentation semantics). 

It has been shown that argumentation is a perspicuous abstraction method capturing several existing reasoning paradigms, particularly non-monotonic reasoning and logic programming (see e.g.~\cite{Dung:1995,Bondarenko:Dung:Kowalski:Toni:1997,Garcia:Simari:2004}), 
as well as forms of decision making (see e.g.~\cite{Amgoud:Prade:2009}), to name a few. 
Consequently, argumentation can be seen as a significant approach to common-sense reasoning. 

Argumentation formalisms can be classified into two families, 
one that builds on Abstract Argumentation (AA) \cite{Dung:1995}, 
and another commonly referred to as structured argumentation (see \cite{Besnard:Garcia:Hunter:Modgil:Prakken:Simari:Toni:2014} for a recent overview). 
Whereas in AA arguments are atomic, structured argumentation formalisms usually specify the internal structure of arguments and/or attacks. 
\emph{Assumption-Based Argumentation} (ABA) \cite{Bondarenko:Dung:Kowalski:Toni:1997,Dung:Mancarella:Toni:2007,Dung:Kowalski:Toni:2006,Dung:Kowalski:Toni:2009,Toni:2013,Toni:2014,Craven:Toni:2016} 
is one particular structured argumentation formalism, 
where knowledge is represented through a deductive system comprising of a formal language and inference rules, 
uncertain information is represented via special sentences in the language, called \emph{assumptions}, 
and attacks are constructively defined among sets of assumptions whenever one set of assumptions deduces (via rules) the \emph{contrary} of some assumption in another set. 

In this paper we deliver a new formalism, called \abap, that extends ABA with \emph{preferences}, 
because accounting for preferences is an important aspect of common-sense reasoning (see e.g.~\cite{Delgrande:Schaub:Tompits:Wang:2004,Brewka:Nemiela:Truszczynski:2008,Kaci:2011,Pigozzi:Tsoukias:Viappiani:2015} for discussions). 
Indeed, preferences are everyday phenomena helping, for example, to qualify the uncertainty of information or to discriminate among conflicting alternatives. 
For an illustration, consider the following example. 

%       example:Referendum
\begin{example}[Referendum]
\label{example:referendum}
At a party, Zed is having a discussion about the outcome of a possible referendum in the Netherlands on whether to remain in the EU. 
Two of his interlocutors, Ann and Bob, have diverging views on the outcome of the referendum.  
Ann claims that the Dutch would vote to leave, whereas Bob maintains that they would vote to stay. 
If this were all the information available, Zed would form two conflicting arguments, based on believing Ann and Bob, respectively. 
However, Zed knows that Ann likes big claims based on dubious assumptions, so he trusts Bob more than Ann. 
This preference information should lead Zed to accepting Bob's argument, rather than Ann's. 
\end{example} 

In ABA (details in section \ref{sec:Background}), Zed's knowledge (without preferences) can be modelled by letting 
$\asma$ and $\asmb$ be assumptions standing for the possibility to trust Ann and Bob, respectively, 
and introducing two rules, $\leave \ot \asma$ and $\stay \ot \asmb$, 
representing the statements of Zed's interlocutors; 
given that Ann and Bob contradict each other, we may set 
$\stay$ and $\leave$ to be the contraries of assumptions $\asma$ and $\asmb$, respectively. 
Thus, $\{ \asma \}$ and $\{ \asmb \}$ are mutually attacking sets of assumptions. 
Intuitively, the preference of $\asmb$ over $\asma$ should lead one to choose $\{ \asmb \}$ over $\{ \asma \}$ as a unique acceptable extension. 
This outcome is obtained in \abap\ by using preference information to \emph{reverse attacks} from an attacker that is less preferred than the attackee: 
due to the preference of $\asmb$ over $\asma$, 
the attack from $\{ \asma \}$ to $\{ \asmb \}$ is reversed into an attack from $\{ \asmb \}$ to $\{ \asma \}$, 
in effect sanctioning $\{ \asmb \}$ as a unique acceptable extension. 

There are many formalisms of argumentation with preferences,  e.g.~\cite{Amgoud:Vesic:2014,Amgoud:Cayrol:2002,Bench-Capon:2003,Kaci:Torre:2008,Modgil:2009,Modgil:Prakken:2010,Baroni:Cerutti:Giacomin:Guida:2011,Brewka:Ellmauthaler:Strass:Wallner:Woltran:2013,Besnard:Hunter:2014,Modgil:Prakken:2014,Garcia:Simari:2014}, 
including approaches extending ABA,  e.g.~\cite{Kowalski:Toni:1996,Fan:Toni:2013,Thang:Luong:2014,Wakaki:2014}. 
This multitude reflects the lack of consensus on how preferences should be accounted for in argumentation in particular, and common-sense reasoning in general (also see e.g.~\cite{Brewka:Nemiela:Truszczynski:2008,Kaci:2011,Domshlak:Hullermeier:Kaci:Prade:2011}). 
A way to discriminate and evaluate distinct formalisms is to investigate their formal properties.  
In this paper we show that \abap\ satisfies various desirable properties of argumentation with preferences. 
We also show that attack reversal differentiates \abap\ from the majority of approaches to argumentation with preferences which discard attacks due to preference information 
(see e.g.~\cite{Amgoud:Cayrol:2002,Bench-Capon:2003,Kaci:Torre:2008,Brewka:Ellmauthaler:Strass:Wallner:Woltran:2013,Besnard:Hunter:2014,Garcia:Simari:2014,Prakken:Sartor:1999,Caminada:Modgil:Oren:2014,Dung:2016}).

The idea of reversing attacks is shared with \emph{Preference-based Argumentation Frameworks} (PAFs) \cite{Amgoud:Vesic:2014}---an AA-based formalism accommodating  preferences. 
PAFs, as well as some structured argumentation formalisms (e.g.~\cite{Besnard:Hunter:2014}), 
assume preferences as given on the meta level, particularly over arguments. 
Instead, \abap\ assumes preferences as given on the object level, particularly over assumptions, 
similarly to the well known structured argumentation formalism \aspicp\ \cite{Modgil:Prakken:2014,Caminada:Amgoud:2007,Prakken:2010,Modgil:Prakken:2013} 
(which, however, accommodates preferences over rules too). 
Most existing approaches assume (e.g.~\cite{Amgoud:Vesic:2014,Besnard:Hunter:2014}) 
or perform (e.g.~\cite{Modgil:Prakken:2014,Wakaki:2014}) 
an aggregation of object-level preferences to give meta-level preferences over arguments (e.g.~\cite{Amgoud:Vesic:2014,Besnard:Hunter:2014,Modgil:Prakken:2014} or extensions (e.g.~\cite{Amgoud:Vesic:2014,Wakaki:2014}). 
Instead, \abap\ incorporates preferences given over assumptions directly into the definition of attack. 

With respect to ABA, existing approaches (notably \cite{Kowalski:Toni:1996,Fan:Toni:2013,Thang:Luong:2014,Wakaki:2014}) have so far accommodated preferences within the class of so-called \emph{flat} ABA frameworks \cite{Bondarenko:Dung:Kowalski:Toni:1997}, where assumptions cannot themselves be deduced from other assumptions. 
Non-flat ABA frameworks are however useful for knowledge representation, as exemplified next. 

%       example:referendum non-flat
\begin{example}[Referendum example extended]
\label{example:referendum non-flat}
Suppose that Dan joins the conversation of Zed, Ann and Bob, 
and says that one should always trust Bob. 
A natural way to model this in ABA is to add an assumption $\asmd$, standing for trust in Dan, as well as a rule $\asmb \ot \asmd$, representing Dan's statement about Bob. 
Such a representation would result into a non-flat ABA framework, because the assumption $\asmb$ is deducible from another assumption, namely $\asmd$. 
\end{example}

\abap\ handles preferences in both flat and non-flat frameworks, thus forgoing possible limitations to expressiveness. 
We show that, in general, \abap\ conservatively extends ABA and satisfies various desirable properties regarding 
relationship among semantics (see e.g.~\cite{Bondarenko:Dung:Kowalski:Toni:1997,Dung:Mancarella:Toni:2007}), 
rationality postulates (see e.g.~\cite{Caminada:Amgoud:2007}), 
and preference handling (see e.g.~\cite{Amgoud:Vesic:2014,Amgoud:Vesic:2009,Brewka:Truszczynski:Woltran:2010,Simko:2014}). 
We also show that flat \abap\ frameworks satisfy additional desirable properties, subject to the principle of \emph{Weak Contraposition} that we propose in this paper, 
and which is a relaxed version of the principle of \emph{contraposition} utilised notably in \aspicp\ to obtain desirable outcomes.

% Structure
The paper is organised as follows. 
We provide background on ABA in section \ref{sec:Background} and present \abap\ in section \ref{sec:ABA+}. 
In section \ref{sec:Properties of ABA+} we study properties of semantics, preference handling and rationality in \abap. 
We introduce Weak Contraposition in section \ref{sec:WCP} and investigate its effects on \emph{flat} \abap\ frameworks in section \ref{sec:Properties of Flat ABA+}. 
We devote section \ref{sec:Comparison} for comparing \abap\ to several formalisms of argumentation with preferences. 
We conclude in section \ref{sec:Concluding Remarks}. 

The precursor to this research is the recently published short paper \cite{Cyras:Toni:2016-KR}, 
where we presented the main idea and basics of \abap. 
With respect to that work, in section \ref{sec:ABA+} we extend \abap\ to non-flat frameworks and define new semantics (ideal). 
Sections \ref{subsec:Relationship Among Semantics} and \ref{subsec:Preference Handling Principles} are new, whereas section \ref{subsec:Rationality Postulates} significantly expands its precursor subsection `Rationality Postulates' in \cite{Cyras:Toni:2016-KR} with new results (starting from Principle \ref{principle:Classical Consistency}). 
The rest of the paper, i.e.~sections \ref{sec:WCP}, \ref{sec:Properties of Flat ABA+} and \ref{sec:Comparison}, consists of new material.

%                                BACKGROUND
\section{Background}
\label{sec:Background}

We base the following ABA background on \cite{Bondarenko:Dung:Kowalski:Toni:1997,Toni:2014}.

%       definition:ABA framework
%\begin{definition}
\label{definition:ABA framework}
An \emph{ABA framework} is a tuple $\abaf$, where:
\begin{itemize}
\item $(\LL, \R)$ is a deductive system with $\LL$ 
a language (a set of sentences) 
and $\R$ a set of rules of the form $\varphi_0 \leftarrow \varphi_1, \ldots, \varphi_m$ 
with $m \geqslant 0$ and $\varphi_i \in \LL$ for $i \in \{ 0, \ldots, m \}$; 
\begin{itemize}
\item $\varphi_0$ is referred to as the \emph{head} of the rule, and
\item $\varphi_1, \ldots, \varphi_m$ is referred to as the \emph{body} of the rule;
\item if $m = 0$, then the rule $\varphi_0 \leftarrow \varphi_1, \ldots, \varphi_m$ 
is said to have an \emph{empty body}, and is written as 
$\varphi_0 \leftarrow \top$, where $\top \not\in \LL$;
\end{itemize}
\item $\A \subseteq \LL$ is a non-empty set, whose elements are referred to as \emph{assumptions};
\item $\contrary : \A \to \LL$ is a total map: 
for $\alpha \in \A$, the $\LL$-sentence $\overline{\alpha}$ is referred to as the \emph{contrary} of $\alpha$.
\end{itemize}
%\end{definition}

In the remainder of this section, unless specified differently, 
we assume as given a fixed but otherwise arbitrary ABA framework \abaf. 

%       definition:deduction
%\begin{definition}
\label{definition:deduction}
A \emph{deduction for $\varphi \in \LL$ supported by $S \subseteq \LL$ and $R \subseteq \R$}, 
denoted by $S \vdash^R \varphi$, 
is a finite tree with 
\begin{itemize}
\item the root labelled by $\varphi$, 
\item leaves labelled by $\top$ or elements from $S$, 
\item the children of non-leaf nodes $\psi$ labelled by the elements of the body of some rule from $\R$ with head $\psi$, and $R$ being the set of all such rules. 
\end{itemize}

For $E \subseteq \LL$, 
the \emph{conclusions} $\cn(E)$ of $E$ is the set of sentences for which deductions supported by subsets of $E$ exist, i.e. 
\begin{align*}
\label{definition:conclusions}
\cn(E) = \{ \varphi \in \LL~:~\exists~S \vdash^R \varphi,~S \subseteq E,~R \subseteq \R \}.
\end{align*}
%\end{definition}

Semantics of ABA frameworks are defined in terms of sets of assumptions meeting desirable requirements. 
One such requirement is being \emph{closed} under deduction, defined as follows. 
%       definition:closure
%\begin{definition}
\label{definition:closure}
For $E \subseteq \LL$, 
the \emph{closure} $\cl(E)$ of $E$ is the set of assumptions that can be deduced from $E$, i.e. 
\begin{align*}
\cl(E) = \{ \asma \in \A~:~\exists~S \vdash^R \asma,~S \subseteq E,~R \subseteq \R \} = \cn(E) \cap \A.
\end{align*}

A set $\asmA \subseteq \A$ of assumptions is \emph{closed} iff $\asmA = \cl(\asmA)$. 
%\end{definition}
We say that \abaf\ is \emph{flat} iff every $\asmA \subseteq \A$ is closed. 
We will see later that flat ABA frameworks exhibit additional properties to those of generic ABA frameworks.

The remaining desirable requirements met by sets of assumptions, 
as semantics for ABA frameworks, 
are given in terms of a notion of \emph{attack} between sets of assumptions, 
defined as follows. 
%       definition:assumption-level attack
%\begin{definition}
\label{definition:assumption-level attack}
A set $\asmA \subseteq \A$ of assumptions \emph{attacks} a set $\asmB \subseteq \A$ of assumptions, 
denoted $\asmA \attacks \asmB$,\footnote{We use the symbol $\attacks$ instead of the commonly used $\to$ to denote attacks in order to avoid confusion, 
when $\to$ (or $\ot$ in the case of ABA) is used to denote rules in structured argumentation, e.g.~\aspicp\ (see section \ref{sec:Comparison}).}
iff there is a deduction $\asmA' \vdash^{R} \contr{\asmb}$, 
for some $\asmb \in \asmB$, 
supported by some $\asmA' \subseteq \asmA$ and $R \subseteq \R$. 
%\end{definition}
If it is not the case that $\asmA$ attacks $\asmB$, then we may write $\asmA \nattacks \asmB$. 
(We will adopt an analogous convention for other attack relations throughout the paper.) 

To define ABA semantics, we use the following auxiliary notions.
%       definition:auxiliary
%\begin{definition}
\label{definition:auxiliary}
For $\asmE \subseteq \A$:
\begin{itemize}
\item $\asmE$ is \emph{conflict-free} iff $\asmE \nattacks \asmE$; 
\item $\asmE$ \emph{defends} $\asmA \subseteq \A$ iff for all closed $\asmB \subseteq \A$ with $\asmB \attacks \asmA$ it holds that $\asmE \attacks \asmB$.\footnote{\label{footnote:defence}Defence in ABA can be equivalently defined `pointwise', 
i.e.~$\asmE$ \emph{defends} $\asmA \subseteq \A$ iff for all $\asma \in \asmA$, for all closed $\asmB \subseteq \A$ with $\asmB \attacks \{ \asma \}$ it holds that $\asmE \attacks \asmB$.}
\end{itemize}
%\end{definition}

ABA semantics are as follows.
%       definition:ABA semantics
%\begin{definition}
\label{definition:ABA semantics} 
A set $\asmE \subseteq \A$ of assumptions (also called an \emph{extension}) is: 
\begin{itemize}
\item \emph{admissible} iff $\asmE$ is closed, conflict-free and defends itself; 
\item \emph{preferred} iff $\asmE$ is $\subseteq$-maximally admissible; 
\item \emph{complete} iff $\asmE$ is admissible and contains every set of assumptions it defends; 
\item \emph{stable} iff $\asmE$ is closed and $\asmE \attacks \{ \asmb \}$ for every $\asmb \in \A \setminus \asmE$; 
\item \emph{well-founded} iff $\asmE$ is the intersection of all complete extensions; 
\item \emph{ideal} iff $\asmE$ is $\subseteq$-maximal among sets of assumptions that are
\begin{itemize}
\item admissible, and 
\item contained in all preferred extensions.
\end{itemize}
\end{itemize}
%\end{definition}

Note that ideal sets of assumptions were originally defined by \citet{Dung:Mancarella:Toni:2007} in the context of flat ABA frameworks only. 
The original definition naturally generalises to any, possibly non-flat, ABA frameworks as given above. 
Note also that, in the case of flat ABA frameworks, 
the term \emph{grounded} is conventionally used instead of \emph{well-founded} 
(e.g.~in \cite{Dung:Mancarella:Toni:2007}): 
we will adopt this convention too later in the paper.

We illustrate various ABA concepts with a formalisation of the Referendum example from the Introduction. 

%       example:background
\begin{example}[Example \ref{example:referendum} as a flat ABA framework]
\label{example:background}
The information given in Example \ref{example:referendum} can be represented as an ABA framework $\F_Z = \abaf$ with 
\begin{itemize}
\item language $\LL = \{ \asma, \asmb, \leave, \stay \}$, 
\item set of rules $\R = \{ \leave \ot \asma, ~~ \stay \ot \asmb \}$, 
\item set of assumptions $\A = \{ \asma, \asmb \}$, 
\item contraries given by: ~ $\contr{\asma} = \stay$, ~ $\contr{\asmb} = \leave$. 
\end{itemize}

Here $\asma$ and $\asmb$ stand for the possibility to trust Ann and Bob, respectively, 
and rules $\leave \ot \asma$ and $\stay \ot \asmb$ represent the statements of Zed's interlocutors. 
Note that $\F_Z$ is flat. 
In $\F_Z$, we find that $\{ \asma \}$ and $\{ \asmb \}$ attack each other, 
and both of them attack and are attacked by $\{ \asma, \asmb \}$, 
which also attacks itself. 
$\F_Z$ can be graphically represented as follows
(in illustrations of ABA frameworks, 
nodes hold sets of assumptions while directed edges indicate attacks): 

\begin{figure}[H]
\caption{Flat ABA framework $\F_Z$}
\begin{center}
\begin{tikzpicture}
%\node at (-1, 0) {$\F_Z$:}; % F_Z
\node at (0.5, 0) {$\emptyset$}; % { }
\draw (0.5, 0) ellipse (0.4 cm and 0.4 cm);
\node at (2, 0) {$\{ \asma \}$}; % a
\draw (2, 0) ellipse (0.4 cm and 0.4 cm);
\node at (4, 0) {$\{ \asmb \}$}; % b
\draw (4, 0) ellipse (0.4 cm and 0.4 cm);
\node at (6.6, 0) {$\{ \asma, \asmb \}$}; % a,b 
\draw (6.6, 0) ellipse (0.6 cm and 0.4 cm);

\draw[symmetric attack] (3.6, 0) to (2.4, 0); % b <-> a
\draw[symmetric attack] (4.4, 0) to (6, 0); % b <-> a, b
\draw[self-attack, out=15, in = 90] (7.2, 0.1) to (8, 0);
\draw[attack, out=270, in = 345] (8, 0) to (7.2, -0.1); % a, b -> a, b
\draw[symmetric attack, out=165, in = 15] (6.07, 0.2) to (2.33, 0.2); % a, b <-> a
\end{tikzpicture}
\end{center}
\end{figure}
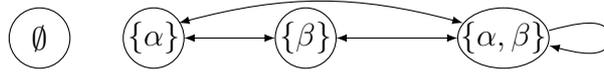

$\F_Z$ has two preferred and stable extensions $\{ \asma \}$ and $\{ \asmb \}$, 
with conclusions $\cn(\{ \asma \}) = \{ \asma, \leave \}$ and $\cn(\{ \asmb \}) = \{ \asmb, \stay \}$, respectively. 
$\F_Z$ has a unique grounded (well-founded) and ideal extension $\emptyset$, 
with conclusions $\cn(\emptyset) = \emptyset$. 
Furthermore, all of $\{ \asma \}$, $\{ \asmb \}$ and $\emptyset$ are admissible and complete extensions. 
\end{example}

Similarly, the extended Referendum example can be represented in ABA, but via a non-flat ABA framework, as follows. 

%       example:non-flat
\begin{example}[Example \ref{example:referendum non-flat} as a non-flat ABA framework]
\label{example:non-flat}
The situation where Dan joins the conversation can be represented by a non-flat ABA framework $\F_D$, 
which is $\F_Z$ from Example \ref{example:background} extended with an additional rule $\asmb \ot \asmd$ 
and an additional assumption $\asmd$ (standing for trust in Dan). 
Overall, $\F_D$ has 
\begin{itemize}
\item $\LL = \{ \asma, \asmb, \asmd, \leave, \stay, \contr{\asmd} \}$,\footnote{Throughout the paper, 
we often slightly abuse notation and, unless specified otherwise, 
simply assume that the contrary $\contr{\asmd}$ of any assumption $\asmd$ 
is actually a symbol in the language $\LL$, 
wherefore we do not need to specify (some part of) the contrary mapping separately.}   
\item $\R = \{ \leave \ot \asma, ~~ \stay \ot \asmb, ~~ \asmb \ot \asmd \}$, 
\item $\A = \{ \asma, \asmb, \asmd \}$, 
\item $\contr{\asma} = \stay$, ~ $\contr{\asmb} = \leave$. 
\end{itemize}

In $\F_D$, if a set of assumptions contains $\asmd$, then it is closed only if it also contains $\asmb$. 
Thus, the only admissible extensions of $\F_D$ are 
$\emptyset$, $\{ \asma \}$, $\{ \asmb \}$ and $\{ \asmb, \asmd \}$. 
Also, $\{ \asmb, \asmd \}$ is a unique stable extension of $\F_D$, 
whereas both $\{ \asma \}$ and $\{ \asmb, \asmd \}$ are preferred. 
$\emptyset$ is thus a unique well-founded and ideal extension of $\F_D$, 
and only $\{ \asmb, \asmd \}$ is complete. 
\end{example}

We also recall, from \cite{Dung:1995}, 
that an AA framework is a pair $\AF$ with a set $\Args$ of \emph{arguments} and a binary \emph{attack} relation $\attacks$ on $\Args$. 
A set $\asmE \subseteq \Args$ of arguments attacks an argument $\argB \in \Args$, written $\asmE \attacks \argB$ in this paper, iff there is $\argA \in \asmE$ with $\argA \attacks \argB$; 
also, $\asmE$ attacks a set $\asmE' \subseteq \Args$ of arguments, written $\asmE \attacks \asmE'$ in this paper, iff there is $\argB \in \asmE'$ with $\asmE \attacks \argB$. 
Then a set $\asmE \subseteq \Args$ is \emph{conflict-free} iff $\asmE \nattacks \asmE$; 
also, $\asmE$ \emph{defends} $\argA \in \Args$ iff for all $\argB \attacks \argA$ we find $\asmE \attacks \argB$. 
Definition of semantics in terms (\flatsemantics, admissible) extensions of AA frameworks is the same as for (flat) ABA frameworks but with `assumptions' replaced by `arguments' (and the closure condition dropped).

%                                ABA+
\section{\abap}
\label{sec:ABA+}

We extend ABA frameworks $\abaf$ with a preference ordering $\leqslant$ on the set $\A$ of assumptions to obtain \emph{\abap\ frameworks $\abafp$}, as follows.

%      definition:ABA+ framework
\begin{definition}
\label{definition:ABA+ framework}
An \textbf{\abap~framework} is a tuple $\abafp$, where $\abaf$ is an ABA framework and $\leqslant$ is a transitive binary relation on $\A$.
\end{definition}

We henceforth apply the notions of conclusions, closure and flatness to \abap\ frameworks, 
having in mind their underlying ABA frameworks. 
The strict counterpart $<$ of $\leqslant$ is defined as 
$\alpha < \beta$ iff $\alpha \leqslant \beta$ and $\beta \nleqslant \alpha$, 
for any $\alpha$ and $\beta$.\footnote{We assume this definition for the strict counterpart of any order used later in this paper, e.g.~for PAFs and \aspicp.} 
The Referendum example can be used to illustrate the concept of an \abap\ framework thus. 

%       example:ABA+ framework
\begin{example}[Example \ref{example:referendum} as a flat \abap\ framework]
\label{example:ABA+ framework}
Recall the ABA framework $\F_Z$ from Example \ref{example:background} representing Zed's knowledge. 
From Example \ref{example:referendum}, we know that Zed trusts Bob more than Ann. 
Hence, we may form a preference over Zed's assumptions, namely $\asma < \asmb$. 
So we obtain an \abap~framework $\F^+_Z = \abafp$ 
with 
\begin{itemize}
\item the underlying ABA framework $\F_Z$ (from Example \ref{example:background}) and
\item the preference ordering $\leqslant$ over $\A$ given by $\asma < \asmb$.
\end{itemize} 
\end{example}

Differently from some other structured argumentation approaches, 
such as for example \aspicp\ \cite{Modgil:Prakken:2014} or DeLP \cite{Garcia:Simari:2014}, 
we consider preferences on assumptions rather than (defeasible) rules. 
This is not, however, a conceptual difference, since assumptions are the only defeasible component in ABA and \abap. 
Also note that, similarly to the approach in \cite{Dung:2016}, 
$\leqslant$ may or may not be a preorder 
(a reflexive and transitive binary relation), 
i.e.~we do not require reflexivity. 

From now on, unless stated differently, 
we consider a fixed but otherwise arbitrary \abap\ framework $\abafp$, 
and implicitly assume $\abaf$ to be its underlying ABA framework. 

We next define the attack relation in \abap. 
The idea is that when the attacker has an assumption less preferred than the one attacked, then the attack is \emph{reversed}.

%       definition:<-attack
\begin{definition}
\label{definition:<-attack}
$\asmA \subseteq \A$ \textbf{$<$-attacks} $\asmB \subseteq \A$, 
denoted $\asmA \pattacks \asmB$, just in case:
\begin{itemize}
\item either there is a deduction $\asmA' \vdash^{R} \contr{\asmb}$, for some $\asmb \in \asmB$, 
supported by $\asmA' \subseteq \asmA$, and $\nexists \asma' \in \asmA'$ with $\asma' < \asmb$;
\item or there is a deduction $\asmB' \vdash^{R} \contr{\asma}$, for some $\asma \in \asmA$, 
supported by $\asmB' \subseteq \asmB$, and $\exists \asmb' \in \asmB'$ with $\asmb' < \asma$.
\end{itemize}

We call an $<$-attack formed as in the first bullet point above a \emph{normal attack},\footnote{Our notion of normal attack is different from the one proposed by \citet{Dung:2016}, which will be discussed in section \ref{subsec:Dung}.} 
and an $<$-attack formed as in the second bullet point above a \emph{reverse attack}. 
\end{definition}

Intuitively, 
$\asmA \pattacks \asmB$ as a normal attack if $\asmA \attacks \asmB$ and no assumption of $\asmA$ used in this attack is strictly less preferred than the attacked assumption (from $\asmB$). 
Otherwise, $\asmB \pattacks \asmA$ as a reverse attack if $\asmA \attacks \asmB$ and this attack depends on at least one assumption that is strictly less preferred than the attacked one. 

%       example:reverse attack
\begin{example}[Attacks in the flat \abap\ framework representing Example \ref{example:referendum}]
\label{example:reverse attack}
Recall the \abap\ framework $\F^+_Z$ from Example \ref{example:ABA+ framework}. 
In $\F^+_Z$, 
$\{ \asma \}$ `tries' to attack $\{ \asmb \}$, 
but is prevented by the preference $\asma < \asmb$. 
Instead, $\{ \asmb \}$ $<$-attacks $\{ \asma \}$ (and also $\{ \asma, \asmb \}$) via reverse attack. 
Likewise, $\{ \asma, \asmb \}$ $<$-attacks both itself and $\{ \asma \}$ via reverse attack. 
$\F^+_Z$ can be represented graphically as follows 
(here and later, \emph{double-tipped arrows} denote attacks that are \emph{both normal and reverse}):

\begin{figure}[H]
\caption{Flat \abap\ framework $\F^+_Z$}
\begin{center}
\begin{tikzpicture}
%\node at (-1, 0) {$\F^+_Z$:}; % F+_Z
\node at (0.5, 0) {$\emptyset$}; % { }
\draw (0.5, 0) ellipse (0.4 cm and 0.4 cm);
\node at (2, 0) {$\{ \asma \}$}; % a
\draw (2, 0) ellipse (0.4 cm and 0.4 cm);
\node at (4, 0) {$\{ \asmb \}$}; % b
\draw (4, 0) ellipse (0.4 cm and 0.4 cm);
\node at (6.6, 0) {$\{ \asma, \asmb \}$}; % a,b 
\draw (6.6, 0) ellipse (0.6 cm and 0.4 cm);

\draw[pattack] (3.6, 0) to (2.4, 0); % b -> a
\draw[pattack] (4.4, 0) to (6, 0); % b -> a, b
\draw[self-pattack, out=15, in = 90] (7.2, 0.1) to (8, 0);
\draw[pattack, out=270, in = 345] (8, 0) to (7.2, -0.1); % a, b -> a, b
\draw[pattack, out=165, in = 15] (6.07, 0.2) to (2.33, 0.2); % a, b -> a
\end{tikzpicture}
\end{center}
\end{figure}

In contrast with the ABA framework $\F_Z$, 
where $\{ \asma \}$ defends against is attackers,  
in the \abap~framework $\F^+_Z$, 
$\{ \asma \}$ is $<$-attacked by, in particular, $\{ \asmb \}$, 
but does not $<$-attack it back. 
This concords with the intended meaning of the preference $\asma < \asmb$, 
that the conflict should be resolved in favour of $\asmb$. 
\end{example}

This concept of $<$-attack reflects the interplay between deductions, contraries and preferences, by representing inherent conflicts among sets of assumptions while accounting for preference information.  
Normal attacks follow the standard notion of attack in ABA, additionally preventing the attack to succeed when the attacker uses assumptions less preferred than the one attacked. 
Reverse attacks, meanwhile, resolve the conflict between two sets of assumptions by favouring the one containing an assumption whose contrary is deduced, 
over the one which uses less preferred assumptions to deduce that contrary. 

We next define the notions of conflict-freeness and defence with respect to $\pattacks$, 
and then introduce \abap\ semantics.

%       definition:<-conflict-freeness
\begin{definition}
\label{definition:<-conflict-freeness}
For $\asmE \subseteq \A$: 
\begin{itemize}
\item $\asmE$ is \textbf{$<$-conflict-free} if $\asmE \npattacks \asmE$; 
\item $\asmE$ \textbf{$<$-defends} $\asmA \subseteq \A$ if for all closed $\asmB \subseteq \A$ with $\asmB \pattacks \asmA$ 
it holds that $\asmE \pattacks \asmB$.
%\footnote{The `pointwise' definition of defence in \abap, i.e.~$\asmE$ \emph{$<$-defends} $\asmA \subseteq \A$ iff for all $\asma \in \asmA$, for all  closed $\asmB \subseteq \A$ with $\asmB \pattacks \{ \asma \}$ it holds that $\asmE \pattacks \asmB$, is not equivalent to the one given above; cf.~footnote \ref{footnote:defence}.} 
\end{itemize} 
%With an abuse of notation we may say that $\asmE$ defends $\asma \in \A$ whenever $\asmE$ defends $\{ \asma \}$. 
\end{definition}

\abap\ semantics can be defined by replacing, in the standard ABA semantics definition, 
the notions of attack and defence with those of $<$-attack and $<$-defence, as follows. 

%       definition:ABA+ semantics
\begin{definition}
\label{definition:ABA+ semantics} 
A set $\asmE \subseteq \A$ of assumptions (also called an \emph{extension}) is: 
\begin{itemize}
\item \textbf{$<$-admissible} if $\asmE$ closed, $<$-conflict-free and $<$-defends itself; 
\item \textbf{$<$-preferred} if $\asmE$ is $\subseteq$-maximally $<$-admissible; 
\item \textbf{$<$-complete} if $\asmE$ is $<$-admissible and contains every set of  assumptions it $<$-defends; 
\item \textbf{$<$-stable}, if $\asmE$ is closed, $<$-conflict-free and $\asmE \pattacks \{ \asmb \}$ for every $\asmb \in \A \setminus \asmE$; 
\item \textbf{$<$-well-founded} if $\asmE$ is the intersection of all $<$-complete extensions; 
\item \textbf{$<$-ideal} if $\asmE$ is $\subseteq$-maximal among sets of assumptions that are
\begin{itemize}
\item $<$-admissible, and 
\item contained in all preferred $<$-extensions.
\end{itemize}
\end{itemize}
\end{definition}

Note: similarly to the convention in ABA, in the case of flat \abap\ frameworks we may use the term \emph{$<$-grounded} instead of $<$-well-founded. 

The following examples illustrate \abap\ semantics. 

%       example:flat ABA+
\begin{example}[Extensions of the flat \abap\ framework representing Example \ref{example:referendum}]
\label{example:flat ABA+}
The flat \abap\ framework $\F^+_Z$ from Example \ref{example:reverse attack} has $<$-admissible extensions $\emptyset$ and $\{ \asmb \}$.  
In particular, $\{ \asma \}$ is not $<$-admissible in $\F^+_Z$ because it does not $<$-defend against, for instance, $\{ \asmb \}$. 
Also, $\{ \asma, \asmb \}$ is not $<$-admissible, because not $<$-conflict-free. 
Hence, $\F^+_Z$ has a unique $<$-complete, $<$-preferred, $<$-stable, $<$-ideal and $<$-grounded extension $\{ \asmb \}$, 
with conclusions $\cn(\{ \asmb \}) = \{ \asmb, \stay \}$. 
\end{example}

%       example:non-flat ABA+
\begin{example}[Example \ref{example:referendum non-flat} as a non-flat \abap\ framework and extensions thereof]
\label{example:non-flat ABA+}
Taking the non-flat ABA framework $\F_D = \abaf$ from Example \ref{example:non-flat}  
and equipping it with preference information $\asma < \asmb$ 
yields a non-flat \abap\ framework $\F^+_D = \abafp$ 
which has a unique $<$-complete, $<$-preferred, $<$-stable, $<$-ideal and $<$-well-founded extension $\{ \asmb, \asmd \}$ with conclusions $\cn(\{ \asmb, \asmd \}) = \{ \asmb, \asmd, \stay \}$. 
\end{example}

Let us consider a slightly more complex setting, 
by way of building on our Referendum example. 

%       example:Carl
\begin{example}[Example \ref{example:referendum} extended and represented as a flat \abap\ framework, and extensions thereof]
\label{example:Carl}
Carl also joins the conversation, and Zed quickly summarises the discussion by saying that 
if one were to believe Ann, it would be `leave', 
whereas if one were to trust Bob, then it would be `stay'. 
Carl does not quite get at first what the conversation is about, 
but swiftly completes the statements (as if they were enthymemes) 
with a missing premise `if there is a referendum'. 
Carl also distrusts Ann, and so believes in the possibility of a referendum and in what Bob is saying more than in what Ann says. 
Arguably, given Carl's preferences, he should also be convinced by Bob rather than Ann.

The information available to Carl can be represented in \abap\ similarly as for Zed in Example \ref{example:background}, 
but with the additional assumption $\asmc$ standing for the possibility of the referendum, 
and the rules $\leave \ot \asma, \asmc$ and $\stay \ot \asmb, \asmc$ instead. 
Overall, Carl's \abap\ framework is $\F^+_C = \abafp$ with
\begin{itemize} 
\item $\LL = \{ \asma, \asmb, \asmc, \leave, \stay, \contr{\asmc} \}$,
\item $\R = \{ \leave \ot \asma, \asmc, ~~ \stay \ot \asmb, \asmc \}$, 
\item $\A = \{ \asma, \asmb, \asmc \}$, 
\item $\contr{\asma} = \stay$, $\contr{\asmb} = \leave$, 
\item $\asma < \asmb, ~~ \asma < \asmc$. 
\end{itemize}

Note that $\F^+_C$ is flat. 
It can be represented graphically thus 
(for readability, we omit the assumption sets $\emptyset$ and $\{ \asma, \asmb, \asmc \}$, 
as well as $<$-attacks to and from them; 
also, here and later, \emph{normal attacks} are denoted by \emph{solid} arrows and \emph{reverse attacks} are denoted by \emph{dotted} arrows; 
as before, double-tipped arrows denote attacks that are both normal and reverse):

\begin{figure}[H]
\caption{Flat \abap\ framework $\F^+_C$}
\begin{center}
\begin{tikzpicture}
%\node at (1, 2) {$\F^+_C$:}; % F+_C
\node at (3, 0.2) {$\{ \asma \}$}; % { a }
\draw (3, 0.2) ellipse (0.5 cm and 0.4 cm);
\node at (1, 1) {$\{ \asmb \}$}; % { b }
\draw (1, 1) ellipse (0.5 cm and 0.4 cm);
\node at (3, 1.8) {$\{ \asmc \}$}; % { c }
\draw (3, 1.8) ellipse (0.5 cm and 0.4 cm);
\node at (5, 0.2) {$\{ \asma, \asmb \}$}; % { a, b }
\draw (5, 0.2) ellipse (0.7 cm and 0.4 cm); 
\node at (7.5, 1) {$\{ \asma, \asmc \}$}; % { a, c }
\draw (7.5, 1) ellipse (0.7 cm and 0.4 cm); 
\node at (5, 1.8) {$\{ \asmb, \asmc \}$}; % { b, c }
\draw (5, 1.8) ellipse (0.7 cm and 0.4 cm); 

\draw [reverse attack] (1.5, 1) to (6.8, 1); % { b } -> { a, c } 
\draw [pattack] (5.7, 1.8) to (7, 1.3); % { b, c } -> { a,c } 
\draw [normal attack] (4.3, 1.7) to (3.3, 0.5); % { b, c } -> { a } 
\draw [normal attack] (5, 1.4) to (5, 0.6); % { b, c } -> { a, b } 
\draw [reverse attack] (5.7, 0.2) to (7, 0.7); % { a, b } -> { a, c }  
\end{tikzpicture}
\end{center}
\end{figure}

The set $\{ \asma, \asmc \}$ (deducing the contrary $\leave$ of $\asmb$) is prevented from $<$-attacking $\{ \asmb \}$, 
and instead $\{ \asmb \}$, as well as any set containing $\asmb$, 
$<$-attacks $\{ \asma, \asmc \}$ via reverse attack. 
Also, $\{ \asmb, \asmc \}$ $<$-attacks $\{ \asma \}$, as well as any set containing $\asma$, 
via normal attack, because no assumption in $\{ \asmb, \asmc \}$ is less preferred than $\asma$. 

The framework $\F^+_C$ has a unique $<$-complete, $<$-preferred, $<$-stable, $<$-ideal and $<$-grounded extension, namely $\{ \asmb, \asmc \}$, 
with conclusions $\cn(\{ \asmb, \asmc \}) = \{ \asmb, \asmc, \stay \}$, 
arguably a desirable outcome. 
\end{example}

Henceforth, we focus on $\sigma \in \{$well-founded/grounded, ideal, stable, preferred, complete$\}$ 
and use $\sigma$ and $<$-$\sigma$ to refer to ABA and \abap\ semantics, respectively. 

We conclude this section with the observations that attacks in ABA can be viewed as $<$-attacks in \abap\ when preferences are absent (Lemma \ref{lemma:attacks coincide}), 
and thus that \abap\ is a conservative extension of ABA (Theorem \ref{theorem:ABA+ extends ABA}). 

%       lemma:attacks coincide
\begin{lemma}
\label{lemma:attacks coincide}
Let $\abafe$ be given. 
For any $\asmA, \asmB \subseteq \A$: $\asmA \attacks \asmB$ iff $\asmA \attacks_{\emptyset} \asmB$.
\end{lemma}

\begin{proof}
Immediate from the definitions of attack in ABA, 
and $<$-attack in \abap, when $\leqslant$ is empty.  
\end{proof}

%       theorem:ABA+ extends ABA
\begin{theorem}
\label{theorem:ABA+ extends ABA}
Let $\abafe$ be given. 
$\asmE \subseteq \A$ is a $\sigma$-extension of $\abaf$ iff 
$\asmE$ is an $\emptyset$-$\sigma$ extension of $\abafe$.
\end{theorem}

\begin{proof}
Immediate from definitions of ABA and \abap\ semantics, and Lemma~\ref{lemma:attacks coincide}.
\end{proof}

Having provided and illustrated the basics of \abap, 
we move on to studying \abap\ in depth. 
In general, argumentation formalisms can be measured against certain principles, 
such as those regarding relationship among semantics (see e.g.~\cite{Dung:1995,Bondarenko:Dung:Kowalski:Toni:1997,Baroni:Giacomin:2007}), 
preference handling (see e.g.~\cite{Kaci:2011,Amgoud:Vesic:2009,Brewka:Truszczynski:Woltran:2010,Amgoud:Vesic:2010}),  
rationality (see e.g.~\cite{Caminada:Amgoud:2007,Modgil:Prakken:2013}) 
and other features of argumentation frameworks (see e.g.~\cite{Dung:2016,Baroni:Giacomin:Liao:2015-IJCAI,Dung:2016-COMMA}). 
We investigate what principles \abap\ adheres to next. 
In particular, we study generic \abap\ frameworks in section \ref{sec:Properties of ABA+}, 
and in section \ref{sec:Properties of Flat ABA+} we analyse flat \abap\ frameworks in the context of Weak Contraposition, which we introduce in section \ref{sec:WCP}.

%                                 PROPERTIES OF ABA+
\section{Properties of Generic \abap\ Frameworks}
\label{sec:Properties of ABA+}

In this section we give some basic properties of \abap\ frameworks 
and then establish relationships among \abap\ semantics, 
following the relationships established for ABA. 
In addition, we analyse some preference handling properties directly applicable to \abap. 
Still further, we investigate the rationality principles \cite{Caminada:Amgoud:2007,Modgil:Prakken:2013} in \abap. 
Results in this section apply to generic, i.e.~both flat and non-flat, \abap\ frameworks.

%                                Basic Properties
\subsection{Basic Properties}
\label{subsec:Basic Properties}

We begin with several basic properties that \abap\ exhibits. 
First, the attack relation in \abap\ is monotonic with respect to set inclusion, like in ABA, as indicated next.

%      lemma:attacks on supersets
\begin{lemma}
\label{lemma:attacks on supersets}
Let $\asmA' \subseteq \asmA \subseteq \A$ and $\asmB' \subseteq \asmB \subseteq \A$ be given. 
If $\asmA' \pattacks \asmB'$, then $\asmA \pattacks \asmB$.
\end{lemma}

\begin{proof}
Immediate from the definition of $\pattacks$. 
\end{proof}

Attacks are preserved across ABA and \abap\ in the following sense:

%       lemma:attacks
\begin{lemma}
\label{lemma:attacks}
For any $\asmA, \asmB \subseteq \A$:
\begin{itemize}
\item if $\asmA \attacks \asmB$, then either $\asmA \pattacks \asmB$ or $\asmB \pattacks \asmA$ (or both);
\item if $\asmA \pattacks \asmB$, then either $\asmA \attacks \asmB$ or $\asmB \attacks \asmA$ (or both).
\end{itemize}
\end{lemma}

\begin{proof}
Let $\asmA, \asmB \subseteq \A$ be arbitrary. 
\begin{itemize}
\item Suppose first $\asmA \attacks \asmB$. 
Then $\exists \asmA' \vdash^R \contr{\asmb}$ such that $\asmb \in \asmB$, $\asmA' \subseteq \asmA$, and 
\begin{enumerate}[leftmargin=1.6cm, align=right]
\item[either (i)] $\forall \asma' \in \asmA'$ we have $\asma' \not< \asmb$, 
\item[or (ii)] $\exists \asma' \in \asmA'$ with $\asma' < \asmb$.
\end{enumerate}
In case (i), $\asmA' \pattacks \asmB$, and hence $\asmA \pattacks \asmB$, by Lemma \ref{lemma:attacks on supersets}. 
In case (ii), $\{ \asmb \} \pattacks \asmA'$, and hence $\asmB \pattacks \asmA$, by Lemma  \ref{lemma:attacks on supersets} as well.

\item Suppose now $\asmA \pattacks \asmB$. Then 
\begin{enumerate}[leftmargin=1.6cm, align=right]
\item[either (i)] $\exists \asmA' \vdash^R \contr{\asmb}$ such that 
$\asmb \in \asmB$, $\asmA' \subseteq \asmA$ and $\forall \asma' \in \asmA'$ we have $\asma' \not< \asmb$,
\item[or (ii)] $\exists \asmB' \vdash^R \contr{\asma}$ such that 
$\asma \in \asmA$, $\asmB' \subseteq \asmB$ and $\exists \asmb' \in \asmB'$ with $\asmb' < \asma$. 
\end{enumerate}
In case (i), $\asmA' \attacks \{ \asmb \}$, and so $\asmA \attacks \asmB$, 
whereas in case (ii), $\asmB' \attacks \{ \asma \}$, so that $\asmB \attacks \asmA$, 
using Lemma \ref{lemma:attacks on supersets} in both cases. 
\end{itemize}
\end{proof}

As an immediate corollary, conflict is preserved across ABA and \abap\ in the following sense: 

%       theorem:<-conflict-free iff conflict-free
\begin{theorem}
\label{theorem:<-conflict-free iff conflict-free}
$\asmE \subseteq \A$ is conflict-free in $\abaf$ iff $\asmE$ is $<$-conflict-free in $\abafp$.
\end{theorem}

We will use this result to establish other desirable properties of \abap\ frameworks, for instance in sections \ref{subsubsec:Conflict Preservation} and \ref{subsec:Rationality Postulates}.

%                                 Relationship Among Semantics
\subsection{Relationship Among Semantics}
\label{subsec:Relationship Among Semantics}

In terms of relationship among semantics, generic \abap\ frameworks exhibit several features exhibited also by generic ABA frameworks. 
We summarise and prove them next. 
(From now on, proofs omitted in the main body of the paper can be found in Appendix A).

%       theorem:ABA+ properties
\begin{theorem}
\label{theorem:ABA+ properties}
Let $\asmE \subseteq \A$.
\begin{enumerate}[label=(\roman*)]
\item If $\asmE$ is $<$-admissible, then there is a $<$-preferred extension $\asmE'$ such that $\asmE \subseteq \asmE'$. 
\item If $\asmE$ is $<$-stable, then it is $<$-preferred. 
\item If $\asmE$ is $<$-stable, then it is $<$-complete. 
\item If $\asmE$ is $<$-well-founded, then for every $<$-stable extension $\asmE'$ it holds that $\asmE \subseteq \asmE'$. 
\item If $\asmE$ is the intersection of all the $<$-preferred extensions and $\asmE$ is also $<$-admissible, then $\asmE$ is $<$-ideal. 
\item If $\asmE$ is $<$-ideal, then it is not $<$-attacked by any $<$-admissible set of assumptions. 
\item If the empty set $\emptyset$ is closed, then there is a $<$-preferred extension, as well as an $<$-ideal extension. 
\end{enumerate}
\end{theorem}

%                        Preference Handling Principles
\subsection{Preference Handling Principles}
\label{subsec:Preference Handling Principles}

We now consider several desirable properties (proposed in \cite{Amgoud:Vesic:2014,Amgoud:Vesic:2009,Brewka:Truszczynski:Woltran:2010,Simko:2014}) of argumentation formalisms dealing with preferences and their satisfaction in \abap. 
Originally, these properties were defined in the context of AA with preferences and/or Logic Programming with preferences. 
In all sections we appropriately reformulate these properties as principles for \abap. 
(Recall that, unless stated otherwise, \abafp\ is assumed to be a fixed but otherwise arbitrary \abap\ framework.)

%                    Conflict Preservation
\subsubsection{Conflict Preservation}
\label{subsubsec:Conflict Preservation}

The first property, proposed by \citet{Amgoud:Vesic:2009} and \citet{Brewka:Truszczynski:Woltran:2010}, insists that extensions returned after accounting for preferences should be conflict-free with respect to the attack relation not taking into account preferences. 
We formulate it as a principle applicable to \abap\ as follows.

%       principle:Conflict Preservation
\begin{principle}
\label{principle:Conflict Preservation} 
$\abafp$ fulfils \textbf{\conflict} for $<$-$\sigma$ semantics just in case 
for all $<$-$\sigma$ extensions $\asmE \subseteq \A$ of $\abafp$, 
for any $\asma, \asmb \in \A$, 
$\{ \asma \} \attacks \{ \asmb \}$ implies that either $\asma \not\in \asmE$ or $\asmb \not\in \asmE$ (or both). 
\end{principle}

Conflict preservation is guaranteed in \abap\ directly from Theorem~\ref{theorem:<-conflict-free iff conflict-free}:

%      proposition:conflict preservation
\begin{proposition}
\label{proposition:conflict preservation}
$\abafp$ fulfils \conflict\ for any semantics $<$-$\sigma$.
\end{proposition}

\begin{proof}
Let $\asmE$ be a $<$-$\sigma$ extension of $\abafp$.  
Let $\asma, \asmb \in \A$ be such that $\{ \asma \} \attacks \{ \asmb \}$. 
Then $\{ \asma, \asmb \}$ is not conflict-free, and hence not $<$-conflict-free, 
by Theorem~\ref{theorem:<-conflict-free iff conflict-free}. 
If $\asma, \asmb \in \asmE$, then $\asmE$ is not $<$-conflict-free either, which is a contradiction. 
Thus, $\{ \asma, \asmb \} \nsubseteq \asmE$, as required. 
\end{proof}

%                     Empty Preferences
\subsubsection{Empty Preferences}
\label{subsubsec:Empty Preferences}

The second property, taken from \cite{Amgoud:Vesic:2009,Brewka:Truszczynski:Woltran:2010} 
(adapted also from the literature on Logic Programming with Preferences, see e.g.~\cite{Simko:2014} for a discussion), 
insists that if there are no preferences, then the extensions returned using a preference handling mechanism should be the same as those obtained without accounting for preferences. 
We formulate it as a principle applicable to \abap\ as follows. 

%       principle:Empty Preferences
\begin{principle}
\label{principle:Empty Preferences}
$\abafe$ fulfils \textbf{\emptypref} for $\emptyset$-$\sigma$ semantics just in case 
for all $\emptyset$-$\sigma$ extensions $\asmE \subseteq \A$ of $\abafe$, 
$\asmE$ is a $\sigma$ extension of $\abaf$.
\end{principle}

This principle is guaranteed in \abap, given that it is a conservative extension of ABA (Theorem~\ref{theorem:ABA+ extends ABA}): 

%       proposition:empty preferences
\begin{proposition}
\label{proposition:Empty Preferences}
$\abafe$ fulfils \emptypref\ for any semantics $\emptyset$-$\sigma$.
\end{proposition}

\begin{proof}
Immediate from Theorem~\ref{theorem:ABA+ extends ABA}.
\end{proof}

%                     Maximal Elements
\subsubsection{Maximal Elements}
\label{subsubsec:Maximal Elements}

The next property, proposed by \citet{Amgoud:Vesic:2014} in the context of AA with preferences, concerns inclusion in extensions of the `strongest' arguments, 
i.e.~arguments that are maximal with respect to the preference ordering. 
We next reformulate the property to be applicable to \abap.

%       principle:Maximal Elements
\begin{principle}
\label{principle:Maximal Elements}
Suppose that the preference ordering $\leqslant$ of $\abafp$ is total
and further assume that the set 
$M = \{ \asma \in \A~:~\nexists \asmb \in \A$ with $\asma < \asmb \}$ is closed and $<$-conflict-free. 
$\abafp$ fulfils \textbf{\maximal} for $<$-$\sigma$ semantics just in case 
for all $<$-$\sigma$ extensions $\asmE \subseteq \A$ of $\abafp$, 
it holds that $M \subseteq \asmE$.
\end{principle}

For an illustration, consider $\F^+_Z$ from Example \ref{example:ABA+ framework}. 
$\asmb$ is a unique $\leqslant$-maximal element in $\A$, 
and $\{ \asmb \}$ is a unique $<$-$\sigma$ extension of $\F^+_Z$ for any $\sigma$ (see Example \ref{example:flat ABA+}), 
whence $\F^+_Z$ fulfils \maximal\ for any semantics $<$-$\sigma$. 

Our next result shows that, in general, 
this principle is guaranteed in \abap\ for $<$-well-founded, $<$-stable and $<$-complete semantics. 

%       proposition:maximal elements
\begin{proposition}
\label{proposition:maximal elements}
$\abafp$ fulfils \maximal\ for $<$-complete, $<$-stable and $<$-well-founded semantics. 
\end{proposition}

Under $<$-preferred and $<$-ideal semantics \maximal\ can in general be violated, as illustrated next. 

%       example:no <-complete
\begin{example}
\label{example:no <-complete}
Consider \abafp\ with\footnote{Here and later, 
we usually omit $\LL$ and $\, \contrary$, and adopt the following convention: 
$\LL$ consists of all the sentences appearing in $\R$, $\A$ 
and $\{ \contr{\asma}~:~\asma \in \A \}$; 
and, unless $\contr{\asmx}$ appears in either $\A$ or $\R$, 
it is different from the sentences appearing in either $\A$ or $\R$.} 
\begin{itemize}
\item $\R = \{ \contr{\asmb} \ot \asma, \asmc \}$, 
\item $\A = \{ \asma, \asmb, \asmc \}$, 
\item $\asma < \asmb, ~ \asmb \leqslant \asmc, ~ \asmc \leqslant \asmb$.
\end{itemize}

This \abap\ framework is depicted below 
(with $\emptyset$, $\A$ and $<$-attacks to and from $\A$ omitted for readability):

\begin{figure}[H]
\caption{\abap\ framework \abafp}
\begin{center}
\begin{tikzpicture}
\node at (3, 0.2) {$\{ \asma \}$}; % a
\draw (3, 0.2) ellipse (0.5 cm and 0.4 cm);
\node at (1, 1) {$\{ \asmb \}$}; % b
\draw (1, 1) ellipse (0.5 cm and 0.4 cm);
\node at (3, 1.8) {$\{ \asmc \}$}; % c
\draw (3, 1.8) ellipse (0.5 cm and 0.4 cm);
\node at (5, 0.2) {$\{ \asma, \asmb \}$}; % a, b
\draw (5, 0.2) ellipse (0.7 cm and 0.4 cm); 
\node at (7.5, 1) {$\{ \asma, \asmc \}$}; % a, c
\draw (7.5, 1) ellipse (0.7 cm and 0.4 cm); 
\node at (5, 1.8) {$\{ \asmb, \asmc \}$}; % b, c
\draw (5, 1.8) ellipse (0.7 cm and 0.4 cm); 

\draw [reverse attack] (1.5, 1) to (6.8, 1); % b -> a, c
\draw [reverse attack] (5.7, 1.8) to (7, 1.3); % b, c -> a, c
\draw [reverse attack] (5.7, 0.2) to (7, 0.7); % a, b -> a, c
\end{tikzpicture}
\end{center}
\end{figure}

Note that $\asmb$ and $\asmc$ are $\leqslant$-maximal, $\{ \asmb, \asmc \}$ is closed and $<$-conflict-free, 
and yet \abafp\ admits a $<$-preferred extension $\{ \asma, \asmb \}$, 
as well as an $<$-ideal extension $\{ \asmb \}$, 
none of which contains $\{ \asmb, \asmc \}$. 
\end{example}

In section \ref{sec:Properties of Flat ABA+} we will give sufficient conditions for \abap\ frameworks to fulfil \maximal\ for $<$-preferred and $<$-ideal semantics too.

%                                 Rationality Postulates
\subsection{Rationality Postulates}
\label{subsec:Rationality Postulates}

\emph{Rationality postulates} proposed by \citet{Caminada:Amgoud:2007} can be applied to argumentation formalisms. 
They are well studied in, for instance, \aspicp, 
where several conditions needed to satisfy the principles are established (see e.g.~\cite{Modgil:Prakken:2013}). 
Rationality postulates have not been studied in ABA in general 
(but see \cite{Toni:2008} for an analysis with respect to a version of ABA,  
and \cite{Modgil:Prakken:2013} for an analysis of a restricted form of flat ABA). 
We now study these postulates in \abap\ in general. 
In particular, following \cite{Modgil:Prakken:2013}, 
we provide their precise formulations for \abap\ in general, 
as well as for a restricted class of \abap\ frameworks incorporating classical negation. 
We also establish the satisfaction of the postulates in general, 
and delineate conditions under which \abap\ satisfies the postulates for the restricted class of frameworks. 

We define the postulates using the following auxiliary definitions. 

%       definition:consistency
\begin{definition}
\label{definition:consistency}
$S \subseteq \LL$ is: 
\begin{itemize}
\item \textbf{directly consistent} if there are no $\varphi, \psi \in S$ with $\varphi = \contr{\psi}$; 
\item \textbf{indirectly consistent} if $\cn(S)$ is directly consistent.
\end{itemize}
\end{definition}

Theorem~\ref{theorem:<-conflict-free iff conflict-free} implies that $<$-conflict-free sets are (in)directly consistent.

%       lemma:<-conflict-free implies consistency
\begin{lemma}
\label{lemma:<-conflict-free implies consistency}
Any closed and $<$-conflict-free set $\asmE \subseteq \A$ is both directly and indirectly consistent. 
\end{lemma}

\begin{proof}
As $\asmE$ is $<$-conflict-free, it is conflict-free, 
by Theorem~\ref{theorem:<-conflict-free iff conflict-free}. 
Suppose for a contradiction that $\asmE$ is not directly consistent. 
Then there are $\asma, \asmb \in \asmE$ such that $\asma = \contr{\asmb}$. 
But as $\{ \asma \} \vdash^{\emptyset} \asma$ is a deduction 
supported by $\{ \asma \} \subseteq \asmE$, 
we get $\asmE \attacks \asmE$, contradicting conflict-freeness of $\asmE$. 

Likewise, suppose $\asmE$ is not indirectly consistent. 
Then there are $\varphi, \asmb \in \cn(\asmE)$ such that $\varphi = \contr{\asmb}$, 
and as $\asmE$ is closed, $\asmb \in \asmE$. 
But then there is a deduction $\Phi \vdash^R \varphi$ 
supported by some $\Phi \subseteq \asmE$, 
so that $\asmE \attacks \asmE$, which is a contradiction. 
\end{proof}

We next formulate the rationality postulates for \abap.

%       principle:Rationality Postulates
\begin{principle}
\label{principle:Rationality Postulates}
Let $\asmE_1, \ldots, \asmE_n$ be $<$-$\sigma$ extensions of $\abafp$. 
Then $\abafp$ fulfils, for $<$-$\sigma$ semantics, \textbf{the Principle of}
\begin{itemize}
\item \textbf{Closure} if $\cn(\asmE_i) = \cn(\cn(\asmE_i))$~~$\forall i \in \{ 1, \ldots, n \}$;\footnote{The idea of this principle is that conclusions of extensions should be deductively closed with respect to the deductive system $(\LL, \R)$, 
and, in the case of \abap, the conclusions operator $\cn$ is the deductive closure operator; 
in \aspicp~the operator of closure under strict rules is used instead, see \cite{Modgil:Prakken:2013} for details.}
\item \textbf{Consistency} if $\asmE_i$ is directly consistent~~$\forall i \in \{ 1, \ldots, n \}$;
\item \textbf{Indirect Consistency} if $\asmE_i$ is indirectly consistent~~$\forall i \in \{ 1, \ldots, n \}$.
\end{itemize}
\end{principle}

Satisfaction of these principles is guaranteed in \abap, as shown next.

%       theorem:ABA+ fulfils rationality postulates
\begin{theorem}
\label{theorem:ABA+ fulfils rationality postulates}
$\abafp$ fulfils the principles of Closure, Consistency and Indirect Consistency, 
for any semantics $<$-$\sigma$.
\end{theorem}

\begin{proof}
Satisfaction of \closure\ is immediate from the definition of the conclusions operator $\cn$ (Definition~\ref{definition:deduction}), 
and fulfilment of the principles of Direct and Indirect Consistency follows from Lemma \ref{lemma:<-conflict-free implies consistency}.
\end{proof}

We note that \citet{Caminada:Amgoud:2007} originally intended the postulates to account for classical negation. 
Classical negation is not, however, singled-out in \citet{Modgil:Prakken:2013}'s formulations of the postulates. 
Nevertheless, the original intentions can be accounted for by appropriately formulating the principle of `classical consistency' for \abap, 
and by formally describing what is required of \abap\ frameworks to fulfil this principle. 
We do this next, following \cite{Caminada:Amgoud:2007,Modgil:Prakken:2013}. 

For the remainder of this section, we assume the language $\LL$ to be closed under the classical negation operator $\neg$. 
As a shorthand, the \emph{complement} $-\varphi$ of $\varphi \in \LL$ is: 
$\neg\psi$ if $\varphi = \psi$; 
and $\psi$ if $\varphi = \neg\psi$. 

%       principle:Classical Consistency
\begin{principle}
\label{principle:Classical Consistency}
Let $\asmE_1, \ldots, \asmE_n$ be all the $<$-$\sigma$ extensions of $\abafp$. 
Then $\abafp$ fulfils \textbf{\cconsistency} for $<$-$\sigma$ semantics just in case 
for no $\varphi \in \LL$ it holds that both $\varphi \in \cn(\asmE_i)$ and $-\varphi \in \cn(\asmE_i)$, for any $i \in \{ 1, \ldots, n \}$.
\end{principle}

Obviously, if the rules of an \abap\ framework include a sentence and its negation as rules with empty bodies, then \cconsistency\ is violated, as illustrated below.

%       example:consistency
\begin{example}
\label{example:consistency}
Consider \abafe\ with 
\begin{itemize}
\item $\LL = \{ p, \neg p, \asma, \neg\asma \}$, 
\item $\R = \{ p \ot \top, ~~ \neg p \ot \top \}$, 
\item $\A = \{ \asma \}$, ~ $\contr{\asma} = p$. 
\end{itemize}

Then $\emptyset \vdash^{\{ p \ot \top \}} p$ and $\emptyset \vdash^{\{ \neg p \ot \top \}} \neg p$, 
and so $\emptyset$ is a unique $<$-admissible extension of $\abafp$, and has $\cn(\emptyset) = \{ p, \neg p \}$. 
Thus, $\emptyset$ is unique $<$-$\sigma$ extension of \abafp\ for any $\sigma$, 
and therefore, this \abap\ framework violates \cconsistency. 
\end{example}

To avoid such situations, we can impose a restriction---akin to the property of \emph{axiom consistency} from \cite{Modgil:Prakken:2013}---on \abap\ frameworks, as follows.

%       axiom:Consistency
\begin{axiom}
\label{axiom:Consistency}
$\abafp$ satisfies \textbf{\consistency} just in case for no $\varphi \in \LL$ 
there are deductions $\emptyset \vdash^{R} \varphi$ and $\emptyset \vdash^{R'} -\varphi$, for any $R, R' \subseteq \R$.
\end{axiom}

Clearly, $\abafp$ from Example \ref{example:consistency} does not satisfy \consistency.

We now propose a property of \abap\ frameworks whose satisfaction, together with \consistency, leads to fulfilment of \cconsistency. 

%       axiom:Negation
\begin{axiom}
\label{axiom:Negation}
$\abafp$ satisfies \textbf{\negation} just in case 
for all $\asmA \subseteq \A$, $R \subseteq \R$ and $\varphi \in \LL$ 
it holds that if $\asmA \vdash^R \varphi$ and $\asmA \neq \emptyset$, 
then for some $\asma \in \asmA$ it holds that $\contr{\asma} = -\varphi$.
\end{axiom}

The axiom of Negation essentially requires that if an assumption can be used to derive a sentence, then the negation of that sentence should be the contrary of that assumption. 
Note that this axiom is somewhat restrictive in that it forces the contrary of some assumption to be a particular sentence. 
However, this syntactic restriction is not a semantic restriction, because if another sentence, say $\psi$, 
is wanted as the contrary of $\asma$, 
then rules $-\varphi \ot \psi$ and $\psi \ot -\varphi$ can be added to the framework. 
Another possibility would be to have a more general contrary mapping~~$\contrary :\A \to \wp(\LL)$ 
which assigns a set of contraries to each assumption, 
just like in some formulations of ABA (see e.g.~\cite{Toni:2014,Fan:Toni:2013}) 
equivalent to the standard presentation we adopt in this paper. 
Alternative formulations of \negation\ are beyond the scope of this paper, and are left as future work.

Satisfaction of the axioms of Consistency and Negation guarantees fulfilment of \cconsistency, as our next result shows.

%       proposition:negation implies classical consistency
\begin{proposition}
\label{proposition:negation implies classical consistency}
If $\abafp$ satisfies both axioms of Consistency and Negation, 
then $\abafp$ fulfils \cconsistency\ for any semantics $<$-$\sigma$.
\end{proposition}

\begin{proof}
Fix $\sigma$ and let $\asmE$ be a $<$-$\sigma$ extension of $\abafp$. 
Suppose for a contradiction that for some $\varphi \in \LL$ we have $\varphi, -\varphi \in \cn(\asmE)$. 
Then, by \consistency, there must be deductions $\asmA \vdash^{R} \varphi$ and $\asmB \vdash^{R'} -\varphi$ 
with at least one of $\asmA, \asmB \subseteq \asmE$ non-empty. 
Say $\asmA \neq \emptyset$. 
Thus, by \negation, we have $\contr{\asma} = -\varphi$, for some $\asma \in \asmA$. 
Thus, $\asmA \cup \asmB$, and as a consequence $\asmE$, is not conflict-free, 
and hence not $<$-conflict-free (by Theorem~\ref{theorem:<-conflict-free iff conflict-free}). 
This is a contradiction. 
Therefore, for no $\varphi \in \LL$ we have $\varphi, -\varphi \in \cn(\asmE)$. 
Thus, as $\sigma$ was arbitrary, $\abafp$ fulfils \cconsistency\ for any semantics $<$-$\sigma$. 
\end{proof}

This result shows, given the class of \abap\ frameworks with classical negation in the language, 
how to satisfy classical consistency. 
In particular, it suffices to ensure that the non-defeasible part of the \abap\ framework is consistent (\consistency), 
and to accordingly incorporate the negation into the contrary mapping. 
(Such an approach was also pursued by \citet{Toni:2008}.) 
Note, however, that this is a different approach than the one indirectly proposed by \citet{Modgil:Prakken:2013} for flat ABA frameworks as instances of \aspicp, 
where, in particular, \emph{contraposition} on rules was suggested (consult section \ref{sec:WCP} and Appendix B). 
The conditions that we identify are different from those proposed in \cite{Modgil:Prakken:2013}, 
because \aspicp\ employs a \emph{contrariness function} (consult Appendix B) 
which is different from the contrary mapping in ABA/\abap. 

\vspace{1cm}

In this section we saw that non-flat \abap\ frameworks exhibit various desirable properties proposed in the literature. 
It is known that flat ABA frameworks exhibit additional properties in terms of relationship among semantics \cite{Bondarenko:Dung:Kowalski:Toni:1997,Dung:Mancarella:Toni:2007}. 
\aspicp\ too adheres to various principles, such as rationality postulates, whenever contraposition on rules is imposed.\footnote{Equivalently, \emph{transposition}, as proposed by \cite{Caminada:Amgoud:2007}, can be used; we focus on contraposition in this paper.} 
In the next section, we propose a relaxed version of contraposition, called \emph{Weak Contraposition}, 
and in section \ref{sec:Properties of Flat ABA+} show that, subject to Weak Contraposition, 
flat \abap\ frameworks exhibit additional desirable properties too.

%                        WEAK CONTRAPOSITION
\section{Weak Contraposition}
\label{sec:WCP}

We have shown (Proposition \ref{proposition:conflict preservation}) that conflict preservation is always guaranteed in \abap. 
In order to ensure conflict preservation in other approaches, notably \aspicp, contraposition can be utilised, as illustrated next.\footnote{Consult Appendix B and e.g.~\cite{Modgil:Prakken:2014,Modgil:Prakken:2013} for details on \aspicp. For our purposes here these details are unnecessary.}

%       example:conflict preservation in ASPIC+
\begin{example}
\label{example:conflict preservation in ASPIC+}
Let $\asma$, $\asmb$ be ordinary premises, 
with contradictories $\neg\asma$ and $\neg\asmb$, respectively, 
and let $\asmb \to \neg\asma$ be a strict rule. 
\aspicp\ (without contraposition) produces arguments 
$\argA = [ \asma ]$, $\argB = [ \asmb ]$ and $\argB' = [ \argB \to \neg\asma ]$, 
and a single attack $\argB' \attacks \argA$. 
Since there are no defeasible rules, 
$\argB'$ is less preferred than $\argA$ (in symbols, $\argB' \prec \argA$), 
with respect to any standard argument comparison principle employed in \aspicp. 
Hence, the attack fails and there are no defeats, so that $\{ \argA, \argB, \argB' \}$ is undefeated, yet self-attacking. 

Contraposition forces to add, 
for instance, a strict rule $\asma \to \neg\asmb$, 
hence giving an additional argument $\argA' = [ \argA \to \neg\asmb ]$ defeating both $\argB$ and $\argB'$ (under any argument comparison principle). 
As a result, the desirable $\{ \argA, \argA' \}$ is obtained as a unique acceptable extension if contraposition is imposed.  
\end{example}

Formally, the principle of contraposition can be expressed in \abap\ as follows. 

%       axiom:Contraposition
\begin{axiom}
\label{axiom:Contraposition}
$\abafp$ satisfies \textbf{\contraposition} just in case 
for all $\asmA \subseteq \A$, $R \subseteq \R$ and $\asmb \in \A$ 
it holds that 
\begin{quote}
\textbf{if} $\asmA \vdash^R \contr{\asmb}$, 
\end{quote}
\begin{quote}
\textbf{then} for every $\asma \in A$, 
there is $\asmA_{\asma} \vdash^{R_{\asma}} \contr{\asma}$, 
for some $\asmA_{\asma} \subseteq (\asmA \setminus \{ \asma \}) \cup \{ \asmb\}$ and $R_{\asma} \subseteq \R$.
\end{quote}
\end{axiom}

This axiom requires that if an assumption plays a role in deducing the contrary of another assumption, 
then it should be possible for the latter to contribute to a deduction of the contrary of the former assumption too. 

In \abap, \contraposition\ is not required to guarantee conflict preservation (as sanctioned by Proposition \ref{proposition:conflict preservation}), 
but restrictions on \abap\ frameworks may be needed to ensure other properties, 
such as existence of $<$-complete extensions, 
which need not be guaranteed in general: 
the \abap\ framework from Example \ref{example:no <-complete} has no $<$-complete extension, 
because all the singletons $\{ \asma \}$, $\{ \asmb \}$ and $\{ \asmc \}$ are $<$-unattacked, 
but $\{ \asma, \asmb, \asmc \}$ is not $<$-conflict-free. 

\iffalse
%       example:contraposition
\begin{example}
\label{example:contraposition}
Recall the \abap\ framework $\F^+_C$ from Example \ref{example:Carl}, 
and suppose that we drop the rule $\stay \ot \asmb, \asmc$ from $\R$ to obtain \abafp\ with 
$\R = \{ \leave \ot \asma, \asmc \}$. 
This \abap\ framework is depicted below 
(with $\emptyset$, $\A$ and $<$-attacks to and from $\A$ are omitted for readability):

\begin{center}
\begin{tikzpicture}
\node at (3, 0.2) {$\{ \asma \}$}; % a
\draw (3, 0.2) ellipse (0.5 cm and 0.4 cm);
\node at (1, 1) {$\{ \asmb \}$}; % b
\draw (1, 1) ellipse (0.5 cm and 0.4 cm);
\node at (3, 1.8) {$\{ \asmc \}$}; % c
\draw (3, 1.8) ellipse (0.5 cm and 0.4 cm);
\node at (5, 0.2) {$\{ \asma, \asmb \}$}; % a, b
\draw (5, 0.2) ellipse (0.7 cm and 0.4 cm); 
\node at (7.5, 1) {$\{ \asma, \asmc \}$}; % a, c
\draw (7.5, 1) ellipse (0.7 cm and 0.4 cm); 
\node at (5, 1.8) {$\{ \asmb, \asmc \}$}; % b, c
\draw (5, 1.8) ellipse (0.7 cm and 0.4 cm); 

\draw [reverse attack] (1.5, 1) to (6.8, 1); % b -> a, c
\draw [reverse attack] (5.7, 1.8) to (7, 1.3); % b, c -> a, c
\draw [reverse attack] (5.7, 0.2) to (7, 0.7); % a, b -> a, c
\end{tikzpicture}
\end{center}
This \abap\ framework has no extensions under, for instance, $<$-complete semantics, 
because all the singletons $\{ \asma \}$, $\{ \asmb \}$ and $\{ \asmc \}$ are $<$-unattacked, 
but $\{ \asma, \asmb, \asmc \}$ is not $<$-conflict-free. 
\end{example}
\fi

We will prove (in section~\ref{sec:Properties of Flat ABA+}) 
that a relaxed version of contraposition, formulated below, 
suffices to guarantee desirable properties, such as, in particular, the so-called Fundamental Lemma 
(see e.g.~\cite[Lemma 10]{Dung:1995}, 
\cite[Theorem 5.7]{Bondarenko:Dung:Kowalski:Toni:1997}) 
and all the properties that follow from it, such as existence of $<$-complete extensions. 

%       axiom:Weak Contraposition
\begin{axiom}
\label{axiom:WCP}
$\abafp$ satisfies \textbf{\wcp} just in case 
for all $\asmA \subseteq \A$, $R \subseteq \R$ and $\asmb \in \A$ it holds that 
\begin{quote}
\textbf{if} $\asmA \vdash^R \contr{\asmb}$ \emph{and} there exists $\asma' \in \asmA$ such that $\asma' < \asmb$, 
\end{quote}
\begin{quote}
\textbf{then}, for \emph{some} $\asma \in \asmA$ which is $\leqslant$-minimal such that $\asma < \asmb$, 
there is $\asmA_{\asma} \vdash^{R_{\asma}} \contr{\asma}$, 
for some $\asmA_{\asma} \subseteq (\asmA \setminus \{ \asma \}) \cup \{ \asmb \}$ and $R_{\asma} \subseteq \R$.
\end{quote}
\end{axiom}

This axiom insists on contraposing only when a deduction involves assumptions less preferred than the one whose contrary is deduced. 
Clearly, satisfaction of \contraposition\ implies satisfaction of \wcp. 
However, satisfaction of \wcp\ does not imply satisfaction of \contraposition, as explained next. 

Consider the \abap\ framework $\F^+_C$ from Example \ref{example:Carl}. 
There is a deduction $\{ \asma, \asmc \} \vdash^{\{ \leave \ot \asma, \asmc \}} \leave$,  where $\leave = \contr{\asmb}$ and $\asma < \asmb$, 
which satisfies the antecedent of \wcp. 
Nonetheless, there is also a deduction $\{ \asmb, \asmc \} \vdash^{\{ \stay \ot \asmb, \asmc \}} \stay$, 
where $\stay = \contr{\asma}$ and $\asmb \not< \asma, ~ \asmc \not< \asma$, 
which satisfies the consequent of \wcp. 
As there are no other deductions that satisfy the antecedent of \wcp, $\F^+_C$ satisfies this axiom.  
However, $\F^+_C$ violates \contraposition, because, for instance, 
the existence of the deduction $\{ \asma, \asmc \} \vdash^{\{ \leave \ot \asma, \asmc \}} \leave$ satisfies the antecedent of \contraposition, 
but there is no deduction $\asmS \vdash^R \contr{\asmc}$ with $\asmS \subseteq \{ \asma, \asmb \}$, 
which is a violation of the consequent of \contraposition. 

In essence, satisfaction of \wcp\ ensures that, 
given a $\subseteq$-minimally non-$<$-conflict-free set $\asmS$ of assumptions, 
some least preferred (i.e.~$\leqslant$-minimal) assumption $\asma \in \asmS$ is $<$-attacked (via normal attack) by the rest of the set (i.e.~$\asmS \setminus \{ \asma \}$). 
Thus, \wcp\ has a similar effect as the \emph{inconsistency resolving} property recently proposed by \cite[Definition 8]{Dung:2016-COMMA}. 
However, the latter does not take preferences into account. 
Formal correspondence between \wcp\ and Dung's inconsistency resolving property is a line of future work. 

Note also that any \abap\ framework $\abafe$ (i.e.~when preference information is absent) automatically satisfies \wcp, without forcing any new rules. 
This is a welcome feature, because, as discussed in \cite{Baroni:Giacomin:Liao:2015-IJCAI}, 
standard contraposition together with general contrariness mappings, 
notably the one in \aspicp, may lead to certain unintended behaviours when preferences are not present. 

In the next section we show that \wcp\ allows flat \abap\ frameworks to retain the relationships among semantics known to hold among semantics of flat ABA frameworks, 
which then allows to extend Proposition \ref{proposition:maximal elements} to the rest of semantics considered in this paper.

%                                 PROPERTIES OF FLAT ABA+
\section{Properties of Flat \abap\ Frameworks Satisfying \wcp}
\label{sec:Properties of Flat ABA+}

In this section, unless stated otherwise, 
we assume as given a flat \abap\ framework $\abafp$ that satisfies \wcp.

We first prove that Weak Contraposition suffices for the Fundamental Lemma to hold in flat \abap:

%       lemma:Fundamental
\begin{lemma}[Fundamental Lemma] 
\label{lemma:Fundamental}
Let $\asmS \subseteq \A$ be $<$-admissible and assume that $\asmS$ $<$-defends $\{ \asma \}, \{ \asma' \} \subseteq \A$. 
Then $\asmS \cup \{ \asma \}$ is $<$-admissible and $<$-defends $\{ \asma' \}$.
\end{lemma}

Note that, without Weak Contraposition, the Fundamental Lemma need not hold: 
in Example \ref{example:no <-complete}, 
$\abafp$ is flat but violates \wcp, 
and $\{ \asmb, \asmc \}$ is $<$-admissible and $<$-defends $\{ \asma \}$ 
(because $\{ \asma \}$ is $<$-unattacked), 
but $\{ \asma, \asmb, \asmc \}$ is not $<$-conflict-free and thus not $<$-admissible. 

Lemma \ref{lemma:Fundamental} implies that, subject to Weak Contraposition, the following additional properties of \abap\ semantics hold for flat \abap\ frameworks, 
mirroring the properties held by flat ABA (as well as AA) frameworks. 

%       theorem:flat ABA+ properties
\begin{theorem}
\label{theorem:flat ABA+ properties}
Let $\F = \abafp$. 
\begin{enumerate}[label=(\roman*)]
\item $\F$ has a $<$-preferred extension. 
\item Every $<$-preferred extension of $\F$ is $<$-complete. 
\item $\F$ has a $<$-complete extension.
\item $\F$ has a unique $<$-grounded extension, which is moreover $<$-complete. 
\item $\F$ has a unique $<$-ideal extension, which is moreover $<$-complete. 
\end{enumerate}
\end{theorem}

Observe that if an \abap\ framework is non-flat, 
satisfaction of \wcp\ does not guarantee the properties above, since \abap\ conservatively extends ABA (Theorem \ref{theorem:ABA+ extends ABA}) and these properties can be falsified for non-flat ABA frameworks \cite{Bondarenko:Dung:Kowalski:Toni:1997}. 

Theorem \ref{theorem:flat ABA+ properties} implies that flat \abap\ frameworks satisfying Weak Contraposition 
fulfil \maximal\ not only for $<$-complete, $<$-stable and $<$-well-founded semantics (Proposition \ref{proposition:maximal elements}), 
but also for $<$-preferred and $<$-ideal semantics:

%       corollary:maximal elements for <-preferred
\begin{corollary}
\label{corollary:maximal elements for <-preferred}
If $\abafp$ is flat and satisfies \wcp, 
then it fulfils \maximal\ for $<$-preferred and $<$-ideal semantics.
\end{corollary}

\begin{proof}
By Theorem \ref{theorem:flat ABA+ properties}(ii, v), $<$-preferred and $<$-ideal extensions are $<$-complete for $\abafp$ flat and satisfying \wcp. 
The claim thus follows from Proposition~\ref{proposition:maximal elements}.
\end{proof}

Investigating whether the class of flat \abap\ frameworks satisfying \wcp\ is the smallest class for which the results (Lemma \ref{lemma:Fundamental}, Theorem \ref{theorem:flat ABA+ properties}, Corollary \ref{corollary:maximal elements for <-preferred}) established in this section hold is left for future work.

%                                 COMPARISON
\section{Comparison}
\label{sec:Comparison}

In this section we compare \abap\ to formalisms of argumentation with preferences most relevant to \abap. 
In particular, we focus on: 
PAFs \cite{Amgoud:Vesic:2014,Amgoud:Vesic:2010,Amgoud:Vesic:2011} as the only other argumentation formalism to use attack reversal (to the best of our knowledge); 
\emph{ABA Equipped with Preferences} (p\_ABA) \cite{Wakaki:2014} as the only other formalism extending ABA with preferences; 
and \aspicp~\cite{Modgil:Prakken:2014,Caminada:Amgoud:2007,Prakken:2010,Modgil:Prakken:2013} as a well known structured argumentation formalism that has been shown to capture flat ABA frameworks as instances. 
Unless stated otherwise, in this section ABA/\abap\ frameworks are assumed to be flat. 
This restriction is imposed because p\_ABA extends (with preferences) flat ABA frameworks; 
\aspicp\ captures flat ABA frameworks; 
and PAFs is an AA-based formalism and AA frameworks can both be instantiated with, and seen as instances of, flat ABA frameworks \cite{Dung:Mancarella:Toni:2007,Toni:2012}.

%                                 PAFs
\subsection{PAFs}
\label{subsec:PAFs}

A Preference-based Argumentation Framework (PAF) $\PAF$ consists of an AA framework $\AF$ equipped with a (partial) preorder $\preccurlyeq$ over arguments, 
which is used to generate a defeat relation, denoted by $\defeats$ in this paper, by reversing attacks from less preferred arguments. 
Formally, given a PAF $\PAF$, the \emph{repaired framework} \cite{Amgoud:Vesic:2014} is an AA framework $\DAF$, 
where $\argA \defeats \argB$ iff either $\argA \attacks \argB$ and $\argA \nprec \argB$, 
or $\argB \attacks \argA$ and $\argB \prec \argA$. 
A set $\asmE \subseteq \Args$ is a $\sigma$ extension of $\PAF$ iff it is a $\sigma$ extension of $\DAF$. 

Following the way ABA frameworks capture AA frameworks \cite{Toni:2012}, we next show how \abap\ frameworks readily capture PAFs. 
In what follows, we assume a PAF $\PAF$ as given, unless stated otherwise. 
We map each argument $\argA \in \Args$ into an assumption $\argA \in \A$, 
together with a new symbol $\contr{\argA}$ for the contrary, 
and map each attack $\argA \attacks \argB$ into a rule $\contr{\argB} \ot \argA$. 
The preference ordering $\preccurlyeq$ is transferred as is to constitute $\leqslant$.  
Formally:

%      definition:ABA+ corresponding to PAF}
\begin{definition}
\label{definition:ABA+ corresponding to PAF}
The \abap\ framework corresponding to PAF \PAF\ is \abafp\ with: 
\begin{itemize}
\item $\LL = \Args \cup \{ \contr{\argA}~:~\argA \in \Args \}$, where $\contr{\argA} \not\in \Args$; 
\item $\R = \{ \contr{\argB} \ot \argA~:~\argA \attacks \argB \}$; 
\item $\A = \Args$; 
\item each $\argA \in \A$ has contrary $\contr{\argA}$; 
\item $\leqslant \, = \, \preccurlyeq$.
\end{itemize}
\end{definition}

Note that the \abap\ framework corresponding to \PAF\ is always flat. 

Trivially, attacks in a given PAF are in a one-to-one correspondence with $<$-attacks in the corresponding \abap\ framework, as follows:

%      lemma:PAF defeat iff <-attack
\begin{lemma}
\label{lemma:PAF defeat iff <-attack}
For $\argA, \argB \in \Args$, 
$\argA \defeats \argB$ iff $\{ \argA \} \pattacks \{ \argB \}$.
\end{lemma}

\iffalse
\begin{proof}
\begin{align*}
\argA \defeats \argB~~&\Leftrightarrow 
&&\text{(i) either } \argA \attacks \argB \text{ and } \argA \not< \argB \\
& &&\text{(ii) or } \argB \attacks \argA \text{ and } \argB < \argA \\
&\Leftrightarrow 
&&\text{(i) either } \exists~\contr{\argB} \ot \argA \in \R \text{ and } \argA \not< \argB \\
& &&\text{(ii) or } \exists~\contr{\argA} \ot \argB \in \R \text{ and } \argB < \argA \\
&\Leftrightarrow &&\{ \argA \} \pattacks \{ \argB \}.
\end{align*}
\end{proof}
\fi

%\iffalse
\begin{proof}
$\argA \defeats \argB$ 
$\Leftrightarrow$ 
(i) either $\argA \attacks \argB$ and $\argA \not< \argB$
(ii) or $\argB \attacks \argA$ and $\argB < \argA$ 
$\Leftrightarrow$ 
(i) either $\exists~\contr{\argB} \ot \argA \in \R$ and $\argA \not< \argB$ 
(ii) or $\exists~\contr{\argA} \ot \argB \in \R$ and $\argB < \argA$ 
$\Leftrightarrow$ 
$\{ \argA \} \pattacks \{ \argB \}$.
\end{proof}
%\fi

From the construction of the \abap\ framework corresponding to \PAF\ and Lemma \ref{lemma:PAF defeat iff <-attack}, 
we obtain the following correspondence result, which says that, 
under any semantics $\sigma$, every PAF is an instance of \abap. 

%      theorem:PAF to ABA+
\begin{theorem}
\label{theorem:PAF to ABA+}
Let $\abafp$ be the corresponding \abap\ framework to \PAF. 
Then $\asmE$ is a $\sigma$ extension of $\PAF$ iff $\asmE$ is a $<$-$\sigma$ extension of $\abafp$.
\end{theorem}

Thus, \abap\ can be seen to generalise PAFs, 
similarly to how ABA generalises AA~\cite{Toni:2012}. 

It is known that flat ABA frameworks are instances of AA frameworks~\cite{Dung:Mancarella:Toni:2007}, 
choosing arguments of the form $\argA:~\asmA \vdash \varphi$ 
(where $\asmA \vdash \varphi$ means there is a deduction $\asmA \vdash^R \varphi$ for some $R \subseteq \R$) 
and attacks $\argA \attacks \argB$, 
for arguments $\argA:~\asmA \vdash \varphi$ and $\argB:~\asmB \vdash \psi$,
whenever $\varphi = \contr{\asmb}$ for some $\asmb \in \asmB$. 
Let us see whether flat \abap\ frameworks are similarly instances of PAFs. 
In order to do this, 
given a preference relation $\leqslant$ on $\A$, 
we can try to define an ordering $\preccurlyeq$ over arguments. 
For example, we can utilise the following orderings taken from \cite{Modgil:Prakken:2013,Young:Modgil:Rodrigues:2016}: 

\begin{itemize}
\item $\argA \argeli \argB$ if $\exists \asma \in \asmA$ such that $\forall \asmb \in \asmB$ we find $\asma < \asmb$; 
\item $\argA \argdeli \argB$ if $\exists \asma \in \asmA \setminus \asmB$ such that $\forall \asmb \in \asmB \setminus \asmA$ we find $\asma < \asmb$; 
\item $\argA \argdemq \argB$ if $\forall \asma \in \asmA$ we find $\asmb \in \asmB$ with $\asma \leqslant \asmb$; 
\item $\argA \argdem \argB$ if $\argA \argdemq \argB$ and $\argB \argndemq \argA$. 
\end{itemize}

The three comparison principles \elitist, \democratic\ and \delitist\ are referred to as \emph{Elitist}, \emph{Democratic} and \emph{Disjoint Elitist} \cite{Young:Modgil:Rodrigues:2016}, respectively. 

The following example shows that, whichever argument ordering above is used, 
the original flat \abap\ frameworks and the resulting PAFs are not semantically equivalent. 

%       example:PAFs
\begin{example}
\label{example:PAFs}
Consider the \abap\ framework \abafp\ with:
\begin{itemize}
\item $\R = \{ \contr{\asme} \ot \asmb, \asmb', \quad 
\contr{\asmb} \ot \asme, \asmb', \quad \contr{\asmb'} \ot \asme, \asmb, \quad 
\contr{\asmb} \ot \asmb, \quad \contr{\asmb'} \ot \asmb', \quad 
\contr{\asma} \ot \asmb, \asmb', \quad 
\contr{\asmb} \ot \asma, \asmb', \quad \contr{\asmb'} \ot \asma, \asmb \}$, 
\item $\A = \{ \asma, \asmb, \asmb', \asme \}$, 
\item $\asmb < \asme$.
\end{itemize}

This \abap\ framework is flat and satisfies \wcp. 
Its sets of assumptions that support deductions, together with $<$-attacks among them, 
can be depicted graphically as follows (highlights are to improve readability):

\begin{figure}[H]
\caption{Flat \abap\ framework \abafp}
\vspace{-0.5cm}
\begin{center}
\begin{tikzpicture}
\node at (0, 1) {$\{ \asme \}$}; % e
\draw [very thick] (0, 1) ellipse (0.5 cm and 0.5 cm);
\node at (1, 2) {$\{ \asme, \asmb \}$}; % e, b
\draw (1, 2) ellipse (0.6 cm and 0.5 cm);
\node at (1, 0) {$\{ \asme, \asmb' \}$}; % e, b'
\draw (1, 0) ellipse (0.6 cm and 0.5 cm);
\node at (3, 2) {$\{ \asmb \}$}; % b
\draw (3, 2) ellipse (0.5 cm and 0.5 cm);
\node at (3, 0) {$\{ \asmb' \}$}; % b'
\draw (3, 0) ellipse (0.5 cm and 0.5 cm);
\node at (5, 2) {$\{ \asma, \asmb \}$}; % a, b
\draw (5, 2) ellipse (0.6 cm and 0.5 cm);
\node at (5, 0) {$\{ \asma, \asmb' \}$}; % a, b'
\draw (5, 0) ellipse (0.6 cm and 0.5 cm);
\node at (7, 1) {$\{ \asmb, \asmb' \}$}; % b, b'
\draw [grayed] (7, 1) ellipse (0.6 cm and 0.6 cm);
\node at (9, 1.5) {$\{ \asma \}$}; % a
\draw [dashed, thick] (9, 1.5) ellipse (0.5 cm and 0.5 cm);

\draw[reverse attack] (0.5, 1) to (6.4, 1); % e -> b, b'

\draw[symmetric normal attack] (1, 0.5) to (1, 1.5); % e, b <-> e, b'
\draw[symmetric normal attack] (5, 0.5) to (5, 1.5); % a, b <-> a, b'
\draw[symmetric normal attack] (1.6, 2) to (2.5, 2); % e, b <-> b
\draw[symmetric normal attack] (1.6, 0) to (2.5, 0); % e, b' <-> b'
\draw[symmetric normal attack] (3.5, 2) to (4.4, 2); % b <-> a, b
\draw[symmetric normal attack] (3.5, 0) to (4.4, 0); % b' <-> a, b'
\draw[symmetric normal attack, out=30, in=150] (1.5, 2.2) to (4.5, 2.2); % e, b <-> a, b
\draw[symmetric normal attack, out=330, in=210] (1.5, -0.2) to (4.5, -0.2); % e, b' <-> a, b'
\draw[symmetric normal attack] (1.5, 1.8) to (4.5, 0.2); % e, b <-> a, b'
\draw[symmetric normal attack] (1.4, 1.7) to (2.8, 0.4); % e, b <-> b'
\draw[symmetric normal attack] (1.5, 0.2) to (4.5, 1.8); % e, b' <-> a, b
\draw[symmetric normal attack] (1.4, 0.3) to (2.8, 1.6); % e, b' <-> b
\draw[symmetric normal attack] (3.4, 1.8) to (4.7, 0.4); % b <-> a, b'
\draw[symmetric normal attack] (3.4, 0.2) to (4.7, 1.6); % b' <-> a, b

\draw[normal self-attack, out=135, in=180] (0.9, 2.5) to (1, 3); 
\draw[normal attack, out=0, in=45] (1, 3) to (1.1, 2.5); % e, b -> e, b
\draw[normal self-attack, out=135, in=180] (2.9, 2.5) to (3, 3); 
\draw[normal attack, out=0, in=45] (3, 3) to (3.1, 2.5); % b -> b
\draw[normal self-attack, out=135, in=180] (4.9, 2.5) to (5, 3); 
\draw[normal attack, out=0, in=45] (5, 3) to (5.1, 2.5); % a, b -> a, b
\draw[normal self-attack, out=225, in=180] (0.9, -0.5) to (1, -1); 
\draw[normal attack, out=0, in=315] (1, -1) to (1.1, -0.5); % e, b' -> e, b'
\draw[normal self-attack, out=225, in=180] (2.9, -0.5) to (3, -1); 
\draw[normal attack, out=0, in=315] (3, -1) to (3.1, -0.5); % b' -> b'
\draw[normal self-attack, out=225, in=180] (4.9, -0.5) to (5, -1); 
\draw[normal attack, out=0, in=315] (5, -1) to (5.1, -0.5); % a, b' -> a, b'

\draw[symmetric normal attack] (6.5, 1.2) to (5.5, 1.8); % b, b' <-> a, b
\draw[symmetric normal attack, out=120, in=60] (6.7, 1.5) to (3.4, 2.3); % b, b' <-> b
\draw[normal attack, out=120, in=60] (6.9, 1.6) to (1.45, 2.3); % b, b' -> e, b
\draw[pattack, out=70, in=100] (1.3, 2.4) to (7, 1.6); % e, b -> b, b'
\draw[symmetric normal attack] (6.5, 0.8) to (5.5, 0.2); % b, b' <-> a, b'
\draw[symmetric normal attack, out=240, in=300] (6.7, 0.5) to (3.4, -0.3); % b, b' <-> b'
\draw[normal attack, out=240, in=300] (6.9, 0.4) to (1.45, -0.3); % b, b' -> e, b'
\draw[pattack, out=290, in=260] (1.3, -0.4) to (7, 0.4); % e, b' -> b, b'

\draw[normal self-attack, out=45, in=90] (7.6, 1.1) to (8.2, 1); 
\draw[normal attack, out=270, in=315] (8.2, 1) to (7.6, 0.9); % b, b' -> b, b'
\draw[normal attack, out=60, in=150] (7.5, 1.3) to (8.5, 1.6); % b, b' -> a
\end{tikzpicture}
\end{center}
\vspace{-1cm}
\end{figure}

In essence, disregarding the other deductions, 
the self-$<$-attacking set $\{ \asmb, \asmb' \}$ deduces the contraries of both $\asma$ and $\asme$, 
and as $\asmb$ is less preferred than $\asme$, 
the set $\{ \asme \}$ $<$-attacks $\{ \asmb, \asmb' \}$ (via reverse attack), 
thus effectively $<$-defending $\{ \asma \}$. 
Overall, $\abafp$ has a unique $<$-complete extension $\asmS = \{ \asme, \asma \}$, 
which is $<$-preferred, $<$-grounded and $<$-ideal 
(by Theorem \ref{theorem:flat ABA+ properties}). 

In ABA (ignoring the preferences), the following arguments (named, for ease of reference) can be obtained: 
$\argE: \{ \asme \} \vdash \asme$; ~ $\argB: \{ \asmb \} \vdash \asmb$; ~
$\argB': \{ \asmb' \} \vdash \asmb'$; ~ $\argA: \{ \asma \} \vdash \asma$; ~
$\argX_{\contr{\asme}}: \{ \asmb, \asmb' \} \vdash \contr{\asme}$; ~
$\argE\argB': \{ \asme, \asmb' \} \vdash \contr{\asmb}$; ~
$\argE\argB: \{ \asme, \asmb \} \vdash \contr{\asmb'}$; ~
$\overline{\argB}: \{ \asmb \} \vdash \contr{\asmb}$; ~
$\overline{\argB'}: \{ \asmb' \} \vdash \contr{\asmb'}$; ~
$\argX_{\contr{\asma}}: \{ \asmb, \asmb' \} \vdash \contr{\asma}$; ~
$\argA\argB': \{ \asma, \asmb' \} \vdash \contr{\asmb}$; ~
$\argA\argB: \{ \asma, \asmb \} \vdash \contr{\asmb'}$. 
These arguments together with attacks among them instantiate an AA framework $\AF$, 
which can be depicted graphically as follows:

\begin{figure}[H]
\caption{AA framework \AF}
\vspace{-1cm}
\begin{center}
\begin{tikzpicture}
\node at (0, 1) {$\argE$}; % E
\draw (0, 1) ellipse (0.5 cm and 0.5 cm);
\node at (1, 2) {$\argE\argB$}; % EB
\draw (1, 2) ellipse (0.6 cm and 0.5 cm);
\node at (1, 0) {$\argE\argB'$}; % EB'
\draw (1, 0) ellipse (0.6 cm and 0.5 cm);
\node at (3, 2) {$\argB$}; % B
\draw (3, 2) ellipse (0.5 cm and 0.5 cm);
\node at (4, 3) {$\overline{\argB}$}; % ¬B
\draw (4, 3) ellipse (0.5 cm and 0.5 cm);
\node at (3, 0) {$\argB'$}; % B'
\draw (3, 0) ellipse (0.5 cm and 0.5 cm);
\node at (4, -1) {$\overline{\argB'}$}; % ¬B'
\draw (4, -1) ellipse (0.5 cm and 0.5 cm);
\node at (5, 2) {$\argA\argB$}; % AB
\draw (5, 2) ellipse (0.6 cm and 0.5 cm);
\node at (5, 0) {$\argA\argB'$}; % AB'
\draw (5, 0) ellipse (0.6 cm and 0.5 cm);
\node at (7, 1) {$\argX_{\contr{\asma}}$}; % X_a
\draw (7, 1) ellipse (0.6 cm and 0.6 cm);
\node at (8, 0) {$\argX_{\contr{\asme}}$}; % X_e
\draw (8, 0) ellipse (0.6 cm and 0.6 cm);
\node at (9, 1.5) {$\argA$}; % A
\draw (9, 1.5) ellipse (0.5 cm and 0.5 cm);

\draw[symmetric normal attack] (1, 0.5) to (1, 1.5); % EB <-> EB'
\draw[symmetric normal attack] (5, 0.5) to (5, 1.5); % AB <-> AB'
\draw[symmetric normal attack] (1.5, 1.8) to (4.5, 0.2); % EB <-> AB'
\draw[normal attack] (1.4, 1.7) to (2.8, 0.4); % EB -> B'
\draw[normal attack, out=280, in=170] (1.3, 1.6) to (3.5, -1); % EB -> ¬B'
\draw[symmetric normal attack] (1.5, 0.2) to (4.5, 1.8); % EB' <-> AB
\draw[normal attack] (1.4, 0.3) to (2.8, 1.6); % EB' -> B
\draw[normal attack, out=80, in=190] (1.3, 0.4) to (3.5, 3); % EB' -> ¬B
\draw[normal attack] (4.7, 0.4) to (3.4, 1.8); % AB' -> B
\draw[normal attack, out=120, in=270] (4.8, 0.45) to (4, 2.5); % AB' -> ¬B
\draw[normal attack] (4.7, 1.6) to (3.4, 0.2); % AB -> B'
\draw[normal attack, out=240, in=90] (4.8, 1.55) to (4, -0.5); % AB -> ¬B'

\draw[normal self-attack, out=135, in=180] (3.9, 3.5) to (4, 4); 
\draw[normal attack, out=0, in=45] (4, 4) to (4.1, 3.5); % ¬B -> ¬B
\draw[normal attack] (3.5, 2.9) to (1.5, 2.3); % ¬B -> EB
\draw[normal attack] (3.6, 2.7) to (3.3, 2.4); % ¬B -> B
\draw[normal attack] (4.4, 2.7) to (4.7, 2.4); % ¬B -> AB
\draw[normal attack, out=350, in=135] (4.5, 2.9) to (6.8, 1.55); % ¬B -> X_a

\draw[normal self-attack, out=225, in=180] (3.9, -1.5) to (4, -2); 
\draw[normal attack, out=0, in=315] (4, -2) to (4.1, -1.5); % ¬B' -> ¬B'
\draw[normal attack] (3.5, -0.9) to (1.5, -0.3); % ¬B' -> EB'
\draw[normal attack] (3.6, -0.7) to (3.3, -0.4); % ¬B' -> B'
\draw[normal attack] (4.4, -0.7) to (4.7, -0.4); % ¬B' -> AB'
\draw[normal attack, out=10, in=225] (4.5, -0.9) to (6.8, 0.45); % ¬B' -> X_a

\draw[symmetric normal attack] (6.45, 1.2) to (5.5, 1.8); % X_a <-> AB
\draw[normal attack, out=70, in=100] (1.3, 2.4) to (7, 1.6); % EB -> X_a
\draw[symmetric normal attack] (6.45, 0.8) to (5.5, 0.2); % X_a <-> AB'
\draw[normal attack, out=290, in=260] (1.3, -0.4) to (7, 0.4); % EB' -> X_a

\draw[normal attack, out=300, in=180] (5.4, 1.6) to (7.4, 0.1); % AB -> X_e
\draw[normal attack] (5.6, 0) to (7.4, 0); % AB' -> X_e
\draw[normal attack, out=0, in=100] (4.5, 3) to (7.9, 0.6); % ¬B -> X_e
\draw[normal attack] (4.5, -1) to (7.45, -0.2); % ¬B' -> X_e
\draw[symmetric normal attack, out=75, in=90] (1.2, 2.45) to (8, 0.6); % EB <-> X_e
\draw[symmetric normal attack, out=290, in=240] (1.2, -0.45) to (7.8, -0.55); % EB' <-> X_e

\draw[normal attack, out=60, in=150] (7.5, 1.3) to (8.5, 1.6); % X_a -> A
\draw[normal attack, out=240, in=240] (7.6, -0.45) to (-0.2, 0.55); % X_e -> E
\end{tikzpicture}
\end{center}
\vspace{-1cm}
\end{figure}

Since the only preference information is $\asmb < \asme$, 
in the three resulting PAFs (employing \argeli, \argdem\ and \argdeli, respectively) it suffices to check only whether attacks on $\argE$ succeed as defeats. 
There is only one such attack, namely $\argX_{\contr{\asme}} \attacks \argE$. 
We find 
$\argX_{\contr{\asme}} \argeli \argE$ and $\argX_{\contr{\asme}} \argdeli \argE$, 
while $\argX_{\contr{\asme}} \argndem \argE$. 
Note that, even though with respect to the Elitist and Disjoint Elitist comparison principles 
the attack $\argX_{\contr{\asme}} \attacks \argE$ is reversed into the defeat 
$\argE \defeats \argX_{\contr{\asme}}$, 
the argument $\argE$ still does not defend $\argA$ in $\DAF$ 
(because $\argX_{\contr{\asma}} \attacks \argA$ and 
neither $\argE \nattacks \argX_{\contr{\asma}}$ 
nor $\argX_{\contr{\asma}} \attacks \argE$), 
so that $\{ \argE \}$ is a unique complete extension of $\DAF$, 
and hence of $\PAF$. 
With respect to the Democratic comparison principle, 
$\argX_{\contr{\asme}} \defeats \argE$, 
so that $\argE$ defends neither $\argA$ nor itself, and thus $\emptyset$ is a unique complete extension of ($\DAF$ and) $\PAF$. 
In any event, we see that acceptable assumptions in \abap\ do not correspond to acceptable arguments in PAFs. 
\end{example}

To summarise, we showed that PAFs can be seen as instances of (flat) \abap\ frameworks. 
The converse is not true in the instantiations of PAFs by using standard ABA arguments and well known argument comparison principles, 
because for these instantiations not all attacks stemming from arguments with the same supporting sets of assumptions are reversed, 
whereas in \abap\ attacks are reversed between sets of assumptions (supporting possibly multiple deductions). 

Whether other instantiations of PAFs with \abap\ are possible, 
is left for future work.

%                                 p_ABA
\subsection{p\_ABA}
\label{subsec:p_ABA}

Following \cite{Wakaki:2014}, 
a p\_ABA framework is a tuple $\paba$ with 
$\abaf$ the underlying ABA framework 
and $\preccurlyeq$ a (partial) preorder over $\LL$.\footnote{For readability, we use the symbol $\preccurlyeq$ to differentiate p\_ABA frameworks from \abap\ frameworks, but in examples to follow we give preferences using $\preccurlyeq$ and $\leqslant$ interchangeably.} 
A \emph{preference relation} (i.e.~preorder) $\sqsubseteq$ over ABA extensions is defined 
(via \cite{Sakama:Inoue:1996}'s criterion for comparing answer sets in Logic Programming with preferences) 
to select the `preferable' extensions, 
called $\mathcal{P}$-extensions, of $\paba$.\footnote{\cite{Wakaki:2014} uses the AA framework corresponding to $\abaf$ and its argument extensions, as in e.g.~\cite{Dung:Kowalski:Toni:2009}. 
Instead, we provide an equivalent definition on the assumption level 
(enabled by results from \cite{Dung:Mancarella:Toni:2007} and \cite{Toni:2012}; see also \cite{Toni:2014}).}
%       definition:extension ordering
%\begin{definition}
\label{definition:extension ordering}
Formally, given a p\_ABA framework $\paba$, let $\E$ be the collection of $\sigma$ extensions of $\abaf$. 
A binary \textbf{preference relation} $\sqsubseteq$ over $\E$ can be defined as follows: 
for $\asmE, \asmE', \asmE'' \in \E$,
\begin{itemize}
\item $\asmE \sqsubseteq \asmE'$ if there is $\varphi \in \cn(\asmE') \setminus \cn(\asmE)$ such that
\begin{itemize}
\item there is $\psi \in \cn(\asmE) \setminus \cn(\asmE')$ with $\psi \preccurlyeq \varphi$ and
\item there is no $\chi \in \cn(\asmE) \setminus \cn(\asmE')$ with $\varphi \prec \chi$;
\end{itemize}
\item $\asmE \sqsubseteq \asmE$; 
\item if $\asmE \sqsubseteq \asmE'$ and $\asmE' \sqsubseteq \asmE''$, then $E\sqsubseteq E''$.
\end{itemize}
%\end{definition}

Then, 
%      definition:p_ABA semantics
%\begin{definition}
\label{definition:p_ABA semantics}
a $\sigma$ extension $\asmE$ of $\abaf$ is a \textbf{$\sigma$ $\mathcal{P}$-extension} of $\paba$ 
iff $\asmE \sqsubseteq \asmE'$ implies $\asmE' \sqsubseteq \asmE$, for any $\sigma$ extension $\asmE'$ of $\abaf$.
%\end{definition}

Observe that p\_ABA merely discriminates among extensions of the underlying ABA framework; 
thus, in Example \ref{example:reverse attack}, 
the unique $\sigma$ extension $\{ \asma \}$ of the underlying ABA framework $\abaf$ is a unique $\sigma$ $\mathcal{P}$-extension of the p\_ABA framework $\paba$, 
disregarding the preference $\asma < \asmb$. 

To see another considerable difference between \abap\ and p\_ABA, let us consider odd cycles. 
These frequently prevent existence of, in particular, stable extensions, and may lead to empty grounded extensions. 
However, when preference information is present, it can break cycles. 
We illustrate with the following variant of the well known example of a `3-cycle'.

%       example:cycle
\begin{example}
\label{example:cycle}
Consider $\R = \{ \contr{\asmb} \ot \asma, \quad 
\contr{\asmc} \ot \asmb, \quad 
\contr{\asma} \ot \asmc \}$ 
and $\A = \{ \asma, \asmb, \asmc \}$ 
with $\asmc < \asmb < \asma$. 
In ABA, ignoring the preferences, there is an odd cycle 
$\{ \asma \} \attacks \{ \asmb \} \attacks \{ \asmc \} \attacks \{ \asma \}$, 
and so no stable extension exists, 
and $\emptyset$ is a unique complete (hence grounded, ideal and preferred) extension. 
Thus, in p\_ABA, no stable $\mathcal{P}$-extensions exist either, 
and $\mathcal{P}$-extensions under other semantics are empty. 
That is, preferences do not really play a role. 
In contrast, \abap\ yields a unique $<$-$\sigma$ extension $\{ \asma \}$. 
The situation is illustrated graphically below.

\begin{figure}[H]
\caption{(p\_)ABA vs \abap. 3-cycle}
\begin{center}
\footnotesize
\begin{tikzpicture}
\node at (-1, 1) {(p\_)ABA}; % p_ABA
\node at (0, 0) {$\{ \asmc \}$}; % { gamma }
\draw (0, 0) ellipse (0.4 cm and 0.4 cm);
\node at (1.4, 0.2) {$\{ \asmb \}$}; % { beta }
\draw (1.4, 0.2) ellipse (0.4 cm and 0.4 cm);
\node at (0.6, 1.2) {$\{ \asma \}$}; % { alpha }
\draw (0.6, 1.2) ellipse (0.4 cm and 0.4 cm);

\draw [attack, out=120, in=210] (-0.2, 0.35) to (0.2, 1.3); % { gamma } -> { alpha }
\draw [attack, out=345, in=75] (1, 1.3) to (1.5, 0.55); % { alpha } -> { beta }
\draw [attack, out=220, in=330] (1.2, -0.15) to (0.25, -0.3); % { beta } -> { gamma }

\hspace{0.8cm}

\node at (3, 1) {\abap}; % ABA+
\node at (4, 0) {$\{ \asmc \}$}; % { gamma }
\draw (4, 0) ellipse (0.4 cm and 0.4 cm);
\node at (5.4, 0.2) {$\{ \asmb \}$}; % { beta }
\draw (5.4, 0.2) ellipse (0.4 cm and 0.4 cm);
\node at (4.6, 1.2) {$\{ \asma \}$}; % { alpha }
\draw (4.6, 1.2) ellipse (0.4 cm and 0.4 cm);

\draw [reverse attack, out=210, in=120] (4.2, 1.3) to (3.8, 0.35); % { alpha } -> { gamma }
\draw [normal attack, out=345, in=75] (5, 1.3) to (5.5, 0.55); % { alpha } -> { beta }
\draw [normal attack, out=220, in=330] (5.2, -0.15) to (4.25, -0.3); % { beta } -> { gamma }
\end{tikzpicture}
\end{center}
\end{figure}
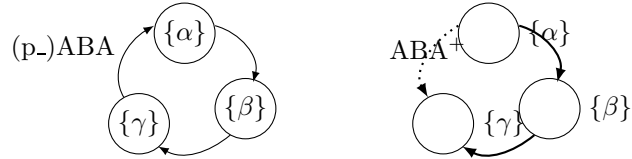
\end{example}

Still further, in settings where stable extensions of the underlying ABA framework exist and preferred extensions are non-empty, 
\abap\ discriminates between extensions differently than p\_ABA, as shown next. 

%       example:p_ABA extensions equivalent
\begin{example}
\label{example:p_ABA extensions equivalent}
Let $\R = \{ \contr{\asma} \ot \asmd, \quad 
\contr{\asmb} \ot \asma, \quad 
\contr{\asmc} \ot \asmb, \quad 
\contr{\asmd} \ot \asmc, \quad 
\contr{\asmc} \ot \asmd \}$, 
$\A = \{ \asma, \asmb, \asmc, \asmd \}$ 
and $\asmd < \asma, ~ \asmc < \asmb$. 
In ABA, ignoring the preferences, we obtain two stable (and also preferred) extensions,  
$\asmE = \{ \asma, \gamma \}$ and $\asmE' = \{ \asmb, \asmd \}$, 
%with $\cn(E) = \{ \alpha, \overline{\beta}, \gamma, \overline{\delta} \}$ and $\cn(E') = \{ \overline{\alpha}, \beta, \overline{\gamma}, \delta \}$, respectively. 
and the grounded/ideal extension $\emptyset$. 
In p\_ABA, both $\asmE$ and $\asmE'$ are stable/preferred $\mathcal{P}$-extensions, 
and $\emptyset$ is the grounded/ideal $\mathcal{P}$-extension. 
In contrast, $\asmE$ is a unique $<$-stable/preferred/ideal extension in \abap, 
and it is in addition $<$-grounded. 
(The situation is illustrated graphically below.) 

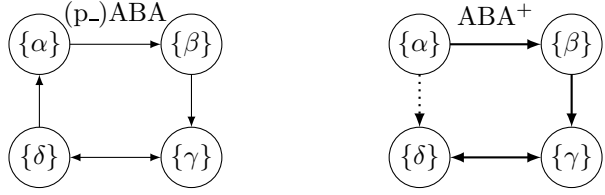
\begin{figure}[H]
\caption{(p\_)ABA vs \abap}
\begin{center}
\footnotesize
\begin{tikzpicture}
\node at (1, 1.9) {(p\_)ABA}; % p_ABA
\node at (0, 0) {$\{ \asmd \}$}; % delta
\draw (0, 0) circle (0.4 cm);
\node at (0, 1.5) {$\{ \asma \}$}; % alpha
\draw (0, 1.5) circle (0.4 cm);
\node at (2, 1.5) {$\{ \asmb \}$}; % beta
\draw (2, 1.5) circle (0.4 cm);
\node at (2, 0) {$\{ \asmc \}$}; % gamma
\draw (2, 0) circle (0.4 cm);

\draw [attack] (0, 0.4) to (0, 1.1); % delta -> alpha
\draw [attack] (0.4, 1.5) to (1.6, 1.5); % alpha -> beta
\draw [attack] (2, 1.1) to (2, 0.4); % beta -> gamma
\draw [symmetric attack] (1.6, 0) to (0.4, 0); % gamma <-> delta

\node at (6, 1.9) {\abap}; % ABA+
\node at (5, 0) {$\{ \asmd \}$}; % delta
\draw (5, 0) circle (0.4 cm);
\node at (5, 1.5) {$\{ \asma \}$}; % alpha
\draw (5, 1.5) circle (0.4 cm);
\node at (7, 1.5) {$\{ \asmb \}$}; % beta
\draw (7, 1.5) circle (0.4 cm);
\node at (7, 0) {$\{ \asmc \}$}; % gamma
\draw (7, 0) circle (0.4 cm);

\draw [reverse attack] (5, 1.1) to (5, 0.4); % alpha -> delta
\draw [normal attack] (5.4, 1.5) to (6.6, 1.5); % alpha -> beta
\draw [normal attack] (7, 1.1) to (7, 0.4); % beta -> gamma
\draw [symmetric normal attack] (6.6, 0) to (5.4, 0); % gamma <-> delta
\end{tikzpicture}
\end{center}
\end{figure}

Given that $\asma$ is objected against only by $\asmd$, 
but due to the preference $\asmd < \asma$ this objection is refuted, 
\abap\ accepts $\asma$, which, arguably, is the correct outcome. 
\end{example}

To summarise, in contrast to \abap, p\_ABA accommodates preferences in order to discriminate among extensions of the underlying ABA framework, which may lead to preference information being ineffective and/or the outcomes unintuitive.

%                                 ASPIC+
\subsection{\aspicp}
\label{subsec:ASPIC+}

\aspicp\ is an expressive argumentation formalism, 
encompassing many key elements of structured argumentation with preferences 
(such as strict and defeasible rules, general contrariness mapping, various forms of attack as well as preferences). 
It was shown by \citet{Prakken:2010} that flat ABA frameworks can be seen as instances of \aspicp\ frameworks (i.e.~as a class of \aspicp\ frameworks without preferences). 
In this section we will show that \abap\ is distinct from \aspicp\ in several respects. 

In order for \aspicp\ frameworks to behave well 
(in the sense of satisfying various formal properties, such as the Fundamental Lemma or the rationality postulates of \cite{Caminada:Amgoud:2007}), 
various requirements have to be met. 
For instance, contraposition, as discussed in section \ref{sec:WCP}, is used; 
preferences also have to satisfy certain conditions. 
The main focus of this section is to exhibit an example which adheres to the standard requirements imposed on \aspicp\ frameworks, 
but is treated differently in \abap\ and \aspicp. 
To this end, we first extend the mapping from ABA to \aspicp\ (without preferences) provided in \cite{Prakken:2010} to a mapping from \abap\ to \aspicp (with preferences). 
%We then consider a particular \abap\ framework that, when thus mapped into an \aspicp\ framework, satisfies the standard \aspicp\ requirements, but is semantically different from its image in \aspicp. 
Details of \aspicp\ sufficient for our purposes in this paper are provided in Appendix B. 

%            ABA+ into ASPIC+
In what follows, unless specified otherwise, we assume as given a flat \abap\ framework \abafp\ with $\leqslant$ a preorder and~~$\contrary:~\A \to \LL$ such that 
$\contr{\asma} = \newcontr{\asma}$, where $\newcontr{\asma} \in \LL \setminus \A$, for any $\asma \in \A$. 

%       definition:ABA+ into ASPIC+
\begin{definition}
\label{definition:ABA+ into ASPIC+}
The \aspicp\ framework corresponding to a flat \abap\ framework \abafp\ is 
$(\overline{\LL}, \ccontrary, \R_s, \K_p, \leqslant_p)$ with: 
\begin{itemize}
\item $\overline{\LL} = \LL \cup \{ \neg s~:~s \in \LL \setminus \A \text{ and } s \neq \newcontr{\asma} \text{ for any } \asma \in \A \}$, where $\neg s \not\in \LL$; 
\item $\ccontrary: \LL \to \wp(\LL)$ is such that:\footnote{For any $\asma \in \A$, $\asma$ and $\newcontr{\asma}$ are contradictories 
(i.e.~$-\asma$ is $\newcontr{\asma}$ and $-\newcontr{\asma}$ is $\asma$); 
and otherwise, $s$ and $\neg s$ are contradictories 
(i.e.~$-s$ is $\neg s$ and $- \neg s$ is $s$). 
We need to use contradictories, because otherwise preferences do not play a role, in the sense that undermining attacks (see Appendix B) always succeed as defeats, whatever the preferences. 
For instance, if in Example \ref{example:conflict preservation in ASPIC+} $\neg\asma$ and $\neg\asmb$ were contraries, rather than contradictories, of the premises $\asma$ and $\asmb$, respectively, 
then the attack $\argB' \attacks \argA$ would succeed as a defeat under any argument comparison principle, irrespective of the preference relation over premises.}
\begin{itemize}
\item if $\asma \in \A$, then $\ccontr{\asma} = \{ \newcontr{\asma} \}$, 
\item if $s = \newcontr{\asma}$ for some $\asma \in \A$, then $\ccontr{s} = \{ \asma \}$, 
\item if $s \in \LL \setminus \A$ and $s \neq \newcontr{\asma}$ for any $\asma \in \A$, then  $\ccontr{s} = \{ \neg s \}$ and $\ccontr{\neg s} = \{ s \}$;
\end{itemize}
\item $\R_s = \{ \varphi_1, \ldots, \varphi_m \to \varphi_0~:~\varphi_0 \ot \varphi_1, \ldots, \varphi_m \in \R \}$; 
\item $\K_p = \A$; 
\item $\leqslant_p \, = \, \leqslant$. 
\end{itemize}
\end{definition}

The following example shows that under this natural mapping, 
\aspicp\ deals with preferences differently from \abap, 
in the sense that the extensions of the given \abap\ framework have conclusions different from the conclusions of the extensions of the \aspicp\ framework obtained via this mapping, 
even though the resulting \aspicp\ framework satisfies the standard \aspicp\ requirements. 

%       example:ASPIC+
\begin{example}
\label{example:ASPIC+}
Recall the \abap\ framework \abafp\ from Example \ref{example:PAFs}. 
Note that it satisfies \contraposition. 
The corresponding \aspicp\ framework $(\overline{\LL}, \ccontrary, \R_s, \K_p, \leqslant_p)$ is thus closed under contraposition (see Appendix B), 
and we obtain the following \aspicp\ arguments 
(similar to those in Example \ref{example:PAFs}, and with the same names, except for $\argX_{-\asma}$ and $\argX_{-\asme}$): 
$\argE = [\asme]$, ~ $\argB = [\asmb]$, ~ $\argB' = [\asmb']$, ~ $\argA = [\asma]$, ~
$\argX_{-\asme} = [\argB, \argB' \to -\asme]$, ~
$\argE\argB' = [\argE, \argB' \to -\asmb]$, ~
$\argE\argB = [\argE, \argB \to -\asmb']$, ~
$\overline{\argB} = [\argB \to -\asmb]$, ~
$\overline{\argB'} = [\argB' \to -\asmb']$, ~
$\argX_{-\asma} = [\argB, \argB' \to -\asma]$, ~
$\argA\argB' = [\argA, \argB' \to -\asmb]$, ~
$\argA\argB = [\argA, \argB \to -\asmb']$. 
The corresponding AA framework $\AF$ can be represented graphically as follows:

\begin{figure}[H]
\caption{AA framework \AF}
\vspace{-1cm}
\begin{center}
\begin{tikzpicture}
\node at (0, 1) {$\argE$}; % E
\draw (0, 1) ellipse (0.5 cm and 0.5 cm);
\node at (1, 2) {$\argE\argB$}; % EB
\draw (1, 2) ellipse (0.6 cm and 0.5 cm);
\node at (1, 0) {$\argE\argB'$}; % EB'
\draw (1, 0) ellipse (0.6 cm and 0.5 cm);
\node at (3, 2) {$\argB$}; % B
\draw (3, 2) ellipse (0.5 cm and 0.5 cm);
\node at (4, 3) {$\overline{\argB}$}; % ¬B
\draw (4, 3) ellipse (0.5 cm and 0.5 cm);
\node at (3, 0) {$\argB'$}; % B'
\draw (3, 0) ellipse (0.5 cm and 0.5 cm);
\node at (4, -1) {$\overline{\argB'}$}; % ¬B'
\draw (4, -1) ellipse (0.5 cm and 0.5 cm);
\node at (5, 2) {$\argA\argB$}; % AB
\draw (5, 2) ellipse (0.6 cm and 0.5 cm);
\node at (5, 0) {$\argA\argB'$}; % AB'
\draw (5, 0) ellipse (0.6 cm and 0.5 cm);
\node at (7, 1) {$\argX_{-\asma}$}; % X_a
\draw (7, 1) ellipse (0.6 cm and 0.6 cm);
\node at (8, 0) {$\argX_{-\asme}$}; % X_e
\draw (8, 0) ellipse (0.6 cm and 0.6 cm);
\node at (9, 1.5) {$\argA$}; % A
\draw (9, 1.5) ellipse (0.5 cm and 0.5 cm);

\draw[symmetric normal attack] (1, 0.5) to (1, 1.5); % EB <-> EB'
\draw[symmetric normal attack] (5, 0.5) to (5, 1.5); % AB <-> AB'
\draw[symmetric normal attack] (1.5, 1.8) to (4.5, 0.2); % EB <-> AB'
\draw[normal attack] (1.4, 1.7) to (2.8, 0.4); % EB -> B'
\draw[normal attack, out=280, in=170] (1.3, 1.6) to (3.5, -1); % EB -> ¬B'
\draw[symmetric normal attack] (1.5, 0.2) to (4.5, 1.8); % EB' <-> AB
\draw[normal attack] (1.4, 0.3) to (2.8, 1.6); % EB' -> B
\draw[normal attack, out=80, in=190] (1.3, 0.4) to (3.5, 3); % EB' -> ¬B
\draw[normal attack] (4.7, 0.4) to (3.4, 1.8); % AB' -> B
\draw[normal attack, out=120, in=270] (4.8, 0.45) to (4, 2.5); % AB' -> ¬B
\draw[normal attack] (4.7, 1.6) to (3.4, 0.2); % AB -> B'
\draw[normal attack, out=240, in=90] (4.8, 1.55) to (4, -0.5); % AB -> ¬B'

\draw[normal self-attack, out=135, in=180] (3.9, 3.5) to (4, 4); 
\draw[normal attack, out=0, in=45] (4, 4) to (4.1, 3.5); % ¬B -> ¬B
\draw[normal attack] (3.5, 2.9) to (1.5, 2.3); % ¬B -> EB
\draw[normal attack] (3.6, 2.7) to (3.3, 2.4); % ¬B -> B
\draw[normal attack] (4.4, 2.7) to (4.7, 2.4); % ¬B -> AB
\draw[normal attack, out=350, in=135] (4.5, 2.9) to (6.8, 1.55); % ¬B -> X_a

\draw[normal self-attack, out=225, in=180] (3.9, -1.5) to (4, -2); 
\draw[normal attack, out=0, in=315] (4, -2) to (4.1, -1.5); % ¬B' -> ¬B'
\draw[normal attack] (3.5, -0.9) to (1.5, -0.3); % ¬B' -> EB'
\draw[normal attack] (3.6, -0.7) to (3.3, -0.4); % ¬B' -> B'
\draw[normal attack] (4.4, -0.7) to (4.7, -0.4); % ¬B' -> AB'
\draw[normal attack, out=10, in=225] (4.5, -0.9) to (6.8, 0.45); % ¬B' -> X_a

\draw[symmetric normal attack] (6.45, 1.2) to (5.5, 1.8); % X_a <-> AB
\draw[normal attack, out=70, in=100] (1.3, 2.4) to (7, 1.6); % EB -> X_a
\draw[symmetric normal attack] (6.45, 0.8) to (5.5, 0.2); % X_a <-> AB'
\draw[normal attack, out=290, in=260] (1.3, -0.4) to (7, 0.4); % EB' -> X_a

\draw[normal attack, out=300, in=180] (5.4, 1.6) to (7.4, 0.1); % AB -> X_e
\draw[normal attack] (5.6, 0) to (7.4, 0); % AB' -> X_e
\draw[normal attack, out=0, in=100] (4.5, 3) to (7.9, 0.6); % ¬B -> X_e
\draw[normal attack] (4.5, -1) to (7.45, -0.2); % ¬B' -> X_e
\draw[symmetric normal attack, out=75, in=90] (1.2, 2.45) to (8, 0.6); % EB <-> X_e
\draw[symmetric normal attack, out=290, in=240] (1.2, -0.45) to (7.8, -0.55); % EB' <-> X_e

\draw[normal attack, out=60, in=150] (7.5, 1.3) to (8.5, 1.6); % X_a -> A
\draw[normal attack, out=240, in=240] (7.6, -0.45) to (-0.2, 0.55); % X_e -> E
\end{tikzpicture}
\end{center}
\vspace{-1cm}
\end{figure}

This $\AF$ has a unique complete extension $\emptyset$ 
(which is thus preferred, ideal and grounded, by results in \cite{Dung:1995,Dung:Mancarella:Toni:2007}). 

As in Example \ref{example:PAFs}, since the only preference information is $\asmb <_p \asme$, it suffices to check only whether attacks on $\argE$ succeed as defeats. 
With respect to the Elitist comparison principle, we obtain $\argX_{-\asme} \argeli \argE$, 
so that $\argX_{-\asme} \ndefeats_{\elitist} \argE, \argE\argB, \argE\argB'$. 
Therefore, $(\Args, \defeats_{\elitist})$ has a unique complete extension $\{ \argE \}$ (which is likewise preferred, ideal and grounded) 
with conclusions (somewhat abusing the notation) 
$\conc(\{ \argE \}) = \bigcup_{\argA \in \{ \argE \}} \conc(\argA) = \{ \asme \}$. 
Note that with respect to the Disjoint Elitist comparison principle, 
we also obtain $\argX_{-\asme} \argdeli \argE$, 
whence $(\Args, \defeats_{\delitist})$ has a unique complete/preferred/ideal/grounded extension $\{ \argE \}$ too. 
With respect to the Democratic comparison, 
$\argX_{-\asme} \argndemq \argE$ because $\asmb' \not\leqslant_p \asme$, 
so that $\defeats_{\democratic} \,=\, \attacks$, 
and so $(\Args, \defeats_{\democratic})$ has a unique complete extension $\emptyset$ 
with $\conc(\emptyset) = \emptyset$. 

As seen in Example \ref{example:PAFs}, \abafp\ has a unique $<$-complete/$<$-preferred/\linebreak$<$-ideal/$<$-grounded extension $\{ \asme, \asma \}$, 
and $\cn(\{ \asme, \asma \}) = \{ \asme, \asma \}$. 
Therefore, under any of the three argument comparison principles, 
conclusions of the unique extension (under any semantics bar ($<$-)stable) of the \abap\ framework and its corresponding \aspicp\ framework are different. 
This happens due to attack reversal in \abap, which allows $\{ \asme \}$ to $<$-defend both itself and $\{ \asma \}$ against $\{ \asmb, \asmb' \}$, because $\asmb < \asme$. 
By contrast, in \aspicp, the two arguments $\argX_{-\asme}$ and $\argX_{-\asma}$ have the same premises $\{ \asmb, \asmb' \}$, 
but only $\argX_{-\asme}$ attacks $\argE$, while $\argX_{-\asma}$ attacks $\argA$. 
Hence, although $\argX_{-\asme} \attacks \argE$ does not result into a defeat due to the  preference $\asmb <_p \asme$, 
the attack $\argX_{-\asma} \attacks \argA$ does result into a defeat. 
In particular, $\argE$ cannot defend $\argA$ against the argument $\argX_{-\asma}$ that has the same premises as the argument $\argX_{-\asme}$ against which $\argE$ defends itself. 

\end{example}

To summarise, it is plain that \abap\ differs conceptually from \aspicp\ in that 
attacks in \abap\ are reversed due to preference information, 
while in \aspicp\ attacks are discarded instead. 
We showed that, as a consequence, even flat \abap\ frameworks satisfying \contraposition\ can yield semantically different outcomes when mapped into \aspicp\ frameworks (in a natural manner). 
This complements the results on the contrasting behaviour of the two formalisms: 
for instance, the desirable ability to unconditionally preserve conflicts in \abap\ (Proposition \ref{proposition:conflict preservation}) 
in contrast to the need to impose certain requirements (such as contraposition) in \aspicp\ (cf.~Example \ref{example:conflict preservation in ASPIC+}); 
also, Weak Contraposition suffices for the Fundamental Lemma \ref{lemma:Fundamental} to hold in \abap, 
whereas \aspicp\ needs (full) contraposition, which is strictly stronger than Weak Contraposition. 

Whether any correspondence is possible under ($<$-)stable semantics, 
and whether other mappings from \abap\ to \aspicp\ would allow to establish a correspondence for some restricted class of frameworks, is a line of future research.

%                                 Dung
\subsection{Dung's Normal Attack}
\label{subsec:Dung}

\citet{Dung:2016} proposed a novel attack relation, called \emph{normal attack}, for \aspicp-type argumentation formalisms. 
This notion of normal attack is presented in a simplified setting, where premises (i.e.~$\K_p$) 
are represented as defeasible rules with empty bodies, 
similar to e.g.~\cite{Caminada:Modgil:Oren:2014}. 
In such a setting, we can attempt to map \abap\ into \aspicp\ as in section \ref{subsec:ASPIC+}, 
but with the following change in Definition \ref{definition:ABA+ into ASPIC+}: 
instead of having premises as assumptions ($\K_p = \A$), we have 
\begin{itemize}
\item set of defeasible rules $\R_d = \{ \To \asma~:~\asma \in \A \}$, 
\item an ordering $\leqslant_d$ on $\R_d$ given by 
$\To \asma~\leqslant_d~\To \asmb$ iff $\asma \leqslant \asmb$, and 
\item $\K_p = \emptyset$ with $\leqslant_p \, = \emptyset$. 
\end{itemize}

For our purposes, normal attack can be defined thus.\footnote{This is the notion of a normal rebut in \cite{Dung:2016}, sufficient for our purposes. 
For more general settings, \emph{undermining} is also included; see \cite{Dung:2016} for details.}
%       definition:Dung's normal attack
%\begin{definition}
\label{definition:Dung's normal attack}
Let $\argA, \argB$ be (\aspicp) arguments. 
$\argA$ \emph{normal-attacks} $\argB$ (at $\argB'$), written $\argA \Nattacks \argB$, iff
$\argB' = [ \To \asmb ] \in \sub(\argB)$, 
$\conc(\argA) = -\asmb$ 
and $\nexists [ \To \asma ] \in \sub(\argA)$ such that $\To \asma~<_d~\To \asmb$.
%\end{definition}
Dung's proposal is, given an \aspicp\ framework \aspicfo, to generate \aspicp\ arguments $\Args$ and then use directly the normal attack relation $\Nattacks$ to construct AA frameworks $\NAF$ and compute extensions. 

With the change in the mapping from \abap\ to \aspicp, and using the normal attack described above, 
we next revisit Example \ref{example:ASPIC+} to illustrate the difference between \abap\ and Dung's approach. 

%       example:Dung
\begin{example}
\label{example:Dung}
Mapping the \abap\ framework \abafp\ from Example \ref{example:ASPIC+} into an \aspicp\ framework as discussed above, 
we obtain arguments 
$\argE = [\To \asme]$, $\argB = [\To \asmb]$, $\argB' = [\To \asmb']$, $\argA = [\To \asma]$, 
and otherwise $\argX_{-\asme}$, $\argE\argB'$, $\argE\argB$, $\overline{\argB}$, $\overline{\argB'}$, 
$\argX_{-\asma}$, $\argA\argB'$, $\argA\argB$ as in Example \ref{example:ASPIC+}. 
The normal attack relation $\Nattacks$ coincides with the Elitist defeat $\defeats_{\elitist}$, 
and so $\NAF$ has a unique complete extension $\{ \argE \}$ with conclusions $\{ \asme \}$. 
Contrasting with \abap, which yields $\{ \asme, \asma \}$ as a unique $<$-complete extension (see Example \ref{example:PAFs}) with conclusions $\{ \asme, \asma \}$, 
we see the same behaviour in Dung's approach as that observed in \aspicp\ in Example \ref{example:ASPIC+}.
\end{example}

To summarise, when focusing on the natural map from \abap\ to \aspicp\ as in this section, 
Dung's normal attack coincides with the Elitist defeat, 
wherefore \abap\ differs from Dung's approach for the same reasons that \abap\ differs from \aspicp. 
Whether any correspondence between \abap\ and \aspicp\ with Dung's normal attack can be established using different mappings is a line of future research.

%                                 Other
\subsection{Other Formalisms}
\label{subsec:Other}

\paragraph{ASPIC$^-$}
\citet{Caminada:Modgil:Oren:2014} proposed a simplified version of \aspicp, 
called ASPIC$^-$, 
with a variation of the \aspicp\ attack relation, called \emph{unrestricted rebut}. 
In ASPIC$^-$, axioms and premises are omitted and instead represented as strict and defeasible rules, respectively. 
The new attack relation is based on the idea that while (standard) rebut in \aspicp\ is allowed only on the conclusion of a defeasible rule, 
unrestricted rebut is allowed on the conclusion of any rule. 
A major feature of ASPIC$^-$ is that so far it works only with total preference orderings over defeasible rules, 
and if a partial ordering is used instead, as in \abap, then the extensions under various semantics need not satisfy the rationality postulates of \cite{Caminada:Amgoud:2007} or other desirable properties, as noted in \cite{Caminada:Modgil:Oren:2014}. 
(This can also be witnessed by analysing Example \ref{example:Dung} using the unrestricted rebut.) 
It would nonetheless be interesting to investigate in the future the relation of \abap\ and a generalisation of ASPIC$^-$ to deal with partial preference orderings. 

\paragraph{Rich PAFs}
\citet{Amgoud:Vesic:2014} further introduced the so-called \emph{Rich PAFs}: 
tuples $\RPAF$, where $\PAF$ is a PAF and $\tleq$ is a \emph{refinement relation}---a  preorder---over extensions of the repaired framework $\DAF$ corresponding to $\PAF$. 
Extensions of $\RPAF$ are $\tleq$-maximal extensions of $\DAF$. 
The authors claim that using suitable refinement relations 
allows to select the (intuitively) preferable extensions, 
by excluding other extensions as unacceptable.  
However, \citet{Kaci:2010} argues that excluding, due to the preferences, some extensions is not a desirable solution, because semantics already provide acceptability conditions. 
Rather, ranking the extensions may be more appropriate. 
A discussion on this topic is, however, beyond the scope of this paper. 
We leave the investigation of how \abap\ relates to Rich PAFs for future work. 

\paragraph{Other}
Several other approaches to argumentation with preferences, 
e.g.~DeLP \cite{Garcia:Simari:2014}, 
an early version of preference-based argumentation frameworks \cite{Amgoud:Cayrol:2002}, 
\emph{Value-based Argumentation} \cite{Bench-Capon:2003,Kaci:Torre:2008}, 
and \emph{Deductive Argumentation} \cite{Besnard:Hunter:2014}, 
use preferences to discard attacks from arguments less preferred than the attackees. 
Similar in spirit are 
\emph{Abstract Dialectical Frameworks} \cite{Brewka:Ellmauthaler:Strass:Wallner:Woltran:2013} 
(where an attack may fail if the attacked argument has a supporting argument that is preferred over the attacker) 
as well as AA-based formalisms representing preferences as attacks on attacks (e.g.~\emph{Extended Argumentation Frameworks} \cite{Modgil:2009}, 
\emph{Argumentation Frameworks with Recursive Attacks} \cite{Baroni:Cerutti:Giacomin:Guida:2011}). 
For reasons similar to those advocated regarding the differences between \abap\ and \aspicp, 
those formalisms are different from \abap\ (see also \cite{Cyras:2016}), 
but precise analysis of any correspondence is left for future research.

%                                CONCLUDING REMARKS
\section{Concluding Remarks}
\label{sec:Concluding Remarks}

We have presented \abap, a structured argumentation formalism that conservatively extends Assumption-Based Argumentation (ABA) with preferences and incorporates a novel technique to reverse attacks due to preference information. 
One important aspect is that \abap\ assumes preferences on the object level (i.e.~over assumptions) and incorporates them directly into the definition of attack, 
rather than assuming preferences on the meta level (e.g.~over arguments). 
A further important aspect is that \abap\ allows for preferences in generic ABA frameworks, 
as opposed to allowing for preferences only in flat ABA frameworks \cite{Bondarenko:Dung:Kowalski:Toni:1997}, as in e.g.~\cite{Kowalski:Toni:1996,Fan:Toni:2013,Thang:Luong:2014,Wakaki:2014}. 

We have shown that \abap\ satisfies various desirable properties regarding 
relationship among semantics (e.g.~\cite{Bondarenko:Dung:Kowalski:Toni:1997,Dung:Mancarella:Toni:2007}), 
rationality postulates (e.g.~\cite{Caminada:Amgoud:2007}) 
and preference handling (e.g.~\cite{Amgoud:Vesic:2014,Amgoud:Vesic:2009,Brewka:Truszczynski:Woltran:2010,Simko:2014}).  
We plan to investigate further properties of \abap, as in e.g.~\cite{Dung:2016,Baroni:Giacomin:2007,Baroni:Giacomin:Liao:2015-IJCAI,Dung:2016-COMMA}. 
Another important contribution of this paper is a new principle, Weak Contraposition, 
that relaxes the principle of contraposition used in e.g.~\aspicp~\cite{Modgil:Prakken:2014} and \cite{Dung:2016}, 
while guaranteeing various additional desirable properties for \abap. 
We plan to investigate which other properties Weak Contraposition guarantees for \abap\ and whether Weak Contraposition can be further relaxed. 

We have seen that \abap\ generalises Preference-based Argumentation Frameworks (PAFs) \cite{Amgoud:Vesic:2014}, 
improves upon Assumption-Based Argumentation Equipped with Preferences (p\_ABA) \cite{Wakaki:2014} 
and differs from the majority of formalisms of argumentation with preferences which discard attacks due to preference information (e.g.~\cite{Amgoud:Cayrol:2002,Bench-Capon:2003,Kaci:Torre:2008,Brewka:Ellmauthaler:Strass:Wallner:Woltran:2013,Besnard:Hunter:2014,Garcia:Simari:2014,Caminada:Modgil:Oren:2014,Dung:2016}), 
particularly \aspicp~\cite{Modgil:Prakken:2014}. 
We aim to analyse more precise relationships of \abap\ to the aforementioned as well as other formalisms of argumentation with preferences (e.g.~\cite{Baroni:Cerutti:Giacomin:Guida:2011,Caminada:Modgil:Oren:2014}). 
In addition, since ABA admits as instances various non-monotonic reasoning formalisms \cite{Bondarenko:Dung:Kowalski:Toni:1997}, 
it would be interesting to study the relationship of \abap\ to those formalisms where they have been extended with preferences. 

Other future work directions include: 
analysing complexity of reasoning problems in \abap, akin to analysis for ABA in  \cite{Dimopoulos:Nebel:Toni:2002,Dunne:2009}; 
studying whether computational mechanisms of ABA's \emph{dispute derivations} \cite{Dung:Mancarella:Toni:2007,Dung:Kowalski:Toni:2006,Toni:2013} can be adapted to \abap; 
relating \abap\ to the version of ABA where sets of arguments are seen as graphs \cite{Craven:Toni:2016}; 
developing tools for computing \abap\ extensions; 
extending the analysis in \cite{Cyras:Toni:2015} of \emph{non-monotonic inference properties} for ABA to \abap; 
further investigating how \abap\ relates to various preference handling principles for non-monotonic reasoning (e.g.~\cite{Simko:2014,Brewka:Eiter:1999}); 
studying whether and how \emph{dynamic preferences} (see e.g.~\cite{Prakken:Sartor:1999,Brewka:Woltran:2010}) can be accommodated in \abap.

\newpage
%                                APPENDIX: Proofs
\section*{Appendix A. Proofs}
\label{appendix:Proofs}

\begin{proof}[Theorem \ref{theorem:ABA+ properties}]
\text{ }
\begin{enumerate}[label=(\roman*)]
\item Let $\asmA_0 \subseteq \asmA_1 \subseteq \ldots$, where $\asmA_0 = \asmE$, 
be an $\subseteq$-increasing sequence of $<$-admissible supersets of $\asmE$. 
Take its upper bound $\asmA = \bigcup_{i \geqslant 0} \asmA_i$ and note that it is $<$-admissible: 
if it were either not closed, not $<$-conflict-free, or did not $<$-defend itself, 
then some finite subset (since deductions are finite) $\asmA' \subseteq \asmA$ would not be 
either closed or $<$-conflict-free, or would not $<$-defend itself. 
Now, by Zorn's Lemma, $\asmE$ has a $\subseteq$-maximally $<$-admissible superset, 
i.e.~a $<$-preferred extension containing $\asmE$. 

\item $\asmE$ is by definition closed and $<$-conflict-free. 
Given that $\asmE \pattacks \{ \asmb \}$ for every $\asmb \in \A \setminus \asmE$, 
it is clear that $\asmE \pattacks \asmB$ for every closed $\asmB \subseteq \A$ such that $\asmB \pattacks \asmE$. 
Hence, $\asmE$ is $<$-admissible. 
Moreover, $\asmE$ is $\subseteq$-maximally $<$-admissible, 
because $\asmE \cup \{ \asmb \} \pattacks \asmE \cup \{ \asmb \}$ for any $\asmb \in \A \setminus \asmE$. 
Thus, $\asmE$ is $<$-preferred, as required. 

\item By (ii) above, $\asmE$ is $<$-admissible. 
Suppose for a contradiction that $\asmE$ $<$-defends $\asmA \subseteq \A$ but $\asmA \nsubseteq \asmE$. 
Then $\asma \not\in \asmE$ for some $\asma \in \asmA$. 
Hence, $\asmE \pattacks \{ \asma \}$, and so $\asmE \pattacks \asmA$, due to stability. 
As $\asmE$ $<$-defends $\asmA$, we find $\asmE \pattacks \asmE$, which is a contradiction. 
Thus, by contradiction, $\asmE$ contains every assumption set it $<$-defends, and so is $<$-complete. 

\item By definition, $\asmE$ is contained in every $<$-complete extension. 
By (iii) above, every $<$-stable extension is $<$-complete. 
Hence, $\asmE$ is contained in every $<$-stable extension.

\item $\asmE$ is $\subseteq$-maximal set of assumptions contained in every $<$-preferred extension. 
Given that $\asmE$ is also $<$-admissible, it is by definition $<$-ideal. 

\item Suppose for a contradiction that $\asmB \subseteq \A$ is $<$-admissible and $\asmB \pattacks \asmE$. 
By (i) above, there is a $<$-preferred extension $\asmA$ such that $\asmB \subseteq \asmA$. 
Then, as $\asmE$ is $<$-ideal, we have $\asmE \subseteq \asmA$, and hence $\asmA \pattacks \asmA$, which is a contradiction. 

\item $\emptyset$, being closed, is $<$-admissible. 
Hence, by (i) above, there is a $<$-preferred extension. 
Thus, the intersection of $<$-preferred extensions exists, 
and so it has a $\subseteq$-maximally $<$-admissible subset, 
i.e.~an $<$-ideal extension. 
\end{enumerate}
\vspace{-0.6cm}
\end{proof}

\begin{proof}[Proposition \ref{proposition:maximal elements}]
Let the preference ordering $\leqslant$ of $\abafp$ be total
and suppose $M = \{ \asma \in \A~:~\nexists \asmb \in \A$ with $\asma < \asmb \}$ is closed and $<$-conflict-free. 
We first show that $M$ is not $<$-attacked. 

To begin with, observe that $M$ cannot be $<$-attacked via reverse attack, because its elements are $<$-maximal in \A. 
So fix $\asma \in M$ and suppose for a contradiction that for some $\asmB \subseteq \A$ it holds that $\asmB \vdash^R \contr{\asma}$ for some $R \subseteq \R$ 
and $\forall \asmb \in \asmB~~\asma \leqslant \asmb$ or $\asmb \nleqslant \asma$. 
Since $\leqslant$ is total, it follows that $\asma \leqslant \asmb~~\forall \asmb \in \asmB$. 
But as $\asma$ is $\leqslant$-maximal, it must also hold that $\asmb \leqslant \asma$, for any $\asmb \in \asmB$. 
From here, we show that $\asmB \subseteq M$. 
Indeed, fix $\asmb \in B$ and assume for a contradiction that $\asmb \not\in M$. 
Then $\exists \asmc \in \A$ such that $\asmb < \asmc$. 
By transitivity of $\leqslant$, we find $\asma < \asmc$, contradicting $\asma$'s $\leqslant$-maximality. 
So we must have $\asmb \in M$, and consequently, $\asmB \subseteq M$. 
But now, since $\asma \in M$, $\asmB \subseteq M$ and $\asmB \pattacks \{ \asma \}$, 
this contradicts $<$-conflict-freeness of $M$. 
Therefore, by contradiction, $M$ is $<$-unattacked. 

If $\abafp$ admits no $<$-complete extensions, then the principle is fulfilled trivially. 
Otherwise, let $\asmE$ be a $<$-complete extension of $\abafp$ and suppose for a contradiction that $M \nsubseteq \asmE$. 
Then $\asmE$ does not $<$-defend $M$. 
This means that $\asmS \pattacks M$ for some $\asmS \subseteq \A$, which is a contradiction. 
Hence, by contradiction, $M \subseteq \asmE$. 
Thus, \abafp\ fulfils \maximal\ for $<$-complete semantics. 

Since by Theorem \ref{theorem:ABA+ properties}(iii) $<$-stable extensions are $<$-complete, 
$\abafp$ fulfils \maximal\ for $<$-stable semantics too. 

Finally, for the $<$-well-founded semantics, 
recall that, by definition, the $<$-well-founded extension is the intersection of all the $<$-complete extensions. 
It follows that $\abafp$ fulfils \maximal\ for $<$-well-founded semantics too. 
\end{proof}

%                        proof of lemma:Fundamental
\begin{proof}[Lemma \ref{lemma:Fundamental} (Fundamental Lemma)]
We first prove that $\asmS \cup \{ \asma \}$ is $<$-admissible. 
If $\asma \in \asmS$, then $\asmS \cup \{ \asma \}$ is trivially $<$-admissible. 
If $\asma \not\in \asmS$, we first show by contradiction that $\asmS \cup \{ \asma \}$ is $<$-conflict-free, 
and then that $\asmS \cup \{ \asma \}$ $<$-defends itself. 

Suppose first that $\asmS \cup \{ \asma \}$ is not $<$-conflict-free. 
Then $\asmS \cup \{ \asma \} \pattacks \asmS \cup \{ \asma \}$ via either (1) normal or (2) reverse attack. 
We show that either leads to $\asmS \pattacks \asmS' \cup \{ \asma \}$ for some $\asmS' \subseteq \asmS$, and then that this leads to a contradiction. 

1.~Suppose $\asmS \cup \{ \asma \} \pattacks \asmS \cup \{ \asma \}$ via normal attack. 
Note that as $\asmS$ is $<$-conflict-free and $<$-defends $\{ \asma \}$, 
the $<$-attack $\asmS \cup \{ \asma \} \pattacks \asmS \cup \{ \asma \}$ must involve $\asma$. 
That is, $\asmS' \cup \{ \asma \} \vdash^{R} \contr{\asmb}$ 
for some $\asmS' \subseteq \asmS$ and $\asmb \in \asmS \cup \{ \asma \}$, 
and $\forall \asms' \in \asmS' \cup \{ \asma \}$ we find $\asms' \not< \asmb$. 
If $\asmb = \asma$, then $\asmS' \cup \{ \asma \} \pattacks \{ \asma \}$, 
and so $\asmS \pattacks \asmS' \cup \{ \asma \}$. 
Else, if $\asmb \in \asmS'$, then $\asmS' \cup \{ \alpha \} \pattacks \asmS$, 
and so $\asmS \pattacks \asmS' \cup \{ \asma \}$ too. 

2.~Suppose $\asmS \cup \{ \asma \} \pattacks \asmS \cup \{ \asma \}$ via reverse attack. 
As in 1., this $<$-attack must involve $\asma$, 
i.e.~$\asmS' \cup \{ \asma \} \vdash^R \contr{\asmb}$ 
for some $\asmS' \subseteq \asmS$ and $\asmb \in \asmS \cup \{ \asma \}$, 
and $\exists \asms' \in \asmS' \cup \{ \asma \}$ such that $\asms' < \asmb$. 
Take $\asms'$ to be $\leqslant$-minimal such. 
If $\asmb \in \asmS$, then $\asmS \pattacks \asmS' \cup \{ \asma \}$. 
Else, if $\asmb = \asma$, then $\asms' \neq \asma$ (by asymmetry of $<$), 
and using \wcp~(WCP henceforth) we find 
$\asmA \vdash^{R'} \contr{\asms'}$ for some $\asmA \subseteq (\asmS' \cup \{ \asma \}) \setminus \{ \asms' \}$, 
so that $\asmS' \cup \{ \asma \} \attacks \asmS$. 
Then, by Lemma \ref{lemma:attacks}, 
either $\asmS' \cup \{ \asma \} \pattacks \asmS$ or $\asmS \pattacks \asmS' \cup \{ \asma \}$, 
yielding $\asmS \pattacks \asmS' \cup \{ \asma \}$ in any event. 

In either case (1) or (2), we obtained 
$\asmS \pattacks \asmS' \cup \{ \asma \}$, 
and as $\asmS$ is $<$-conflict-free and $<$-defends $\{ \asma \}$, 
this $<$-attack must be reverse and involve $\asma$: namely, there is 
$\asmA_1 \cup \{ \asma \} \vdash^{R_1} \contr{\asms_1}$ with 
$\asms_1 \in \asmS$, $\asmA_1 \subseteq \asmS'$, 
and $\exists \asms'_1 \in \asmA_1 \cup \{ \asma \}$ with $\asms'_1 < \asms_1$. 
%
% Without loss of generality take 
Take $\asms'_1$ to be $\leqslant$-minimal such. 
By WCP, there is $\asmS_1 \vdash^{R'_1} \contr{\asms'_1}$ with 
$\asmS_1 \subseteq ((\asmA_1 \cup \{ \asma \}) \setminus \{ \asms'_1 \}) \cup \{ \asms_1 \}$ 
and $\forall \asmx \in \asmS_1~~\asmx \not< \asms'_1$ 
(by $\leqslant$-minimality of $\asms'_1$). 
Note that if $\asms'_1 = \asma$, 
then $\asmS_1 \subseteq \asmA_1 \subseteq \asmS$, 
and so $\asmS \pattacks \{ \asma \}$ via normal attack, 
which cannot happen, because $\asmS$ is $<$-conflict-free and $<$-defends $\{ \asma \}$. 
Thus, $\asms'_1 \neq \asma$ and 
there is $\asmS_1 \cup \{ \asma \} \vdash^{R'_1} \contr{\asms'_1}$ with 
$\asmS_1 \subseteq (\asmA_1 \setminus \{ \asms'_1 \}) \cup \{ \asms_1 \}$ 
and $\forall \asmx \in \asmS_1~~\asmx \not< \asms'_1$ 
(by $\leqslant$-minimality of $\asms'_1$). 
That is, $\asmS_1 \cup \{ \asma \} \pattacks \asmA_1$, 
so we find $\asmS \pattacks \asmS_1 \cup \{ \asma \}$, 
again via reverse attack involving $\asma$: 
$\asmA_2 \cup \{ \asma \} \vdash^{R_2} \contr{\asms_2}$, 
$\asms_2 \in \asmS$, $\asmA_2 \subseteq \asmS_1$, 
and $\exists \asms'_2 \in \asmA_2 \cup \{ \asma \}$ with $\asms'_2 < \asms_2$. 
%
% We again impose $\leqslant$-minimality on $\asms'_2$ and 
For $\leqslant$-minimal such $\asms'_2$, 
(likewise assuming $\asms'_2 \neq \asma$) by WCP we get
$\asmS_2 \cup \{ \asma \} \vdash^{R'_2} \contr{\asms'_2}$ with 
$\asmS_2 \subseteq (\asmA_2 \setminus \{ \asms'_2 \}) \cup \{ \asms_2 \}$ 
and $\forall \asmx \in \asmS_2~~\asmx \not< \asms'_2$.

As deductions are finite and $<$ is asymmetric, 
the procedure described above will eventually exhaust pairs of 
$\asms'_k \in \asmA_k$ and $\asms_k \in \asmS_k$ such that $\asms'_k < \asms_k$, 
so that $\asmS \pattacks \asmS_k \cup \{ \asma \}$ will have to be a normal attack, for some $\asmS_k$. 
This leads to a contradiction to $\asmS$ being $<$-admissible and $<$-defending $\{ \asma \}$.

Hence, by contradiction, $\asmS \cup \{ \asma \}$ is $<$-conflict-free. 

We now want to show that $\asmS \cup \{ \asma \}$ $<$-defends itself. 
So let $\asmB \pattacks \asmS \cup \{ \asma \}$. 
As $\asmS$ is $<$-admissible and $<$-defends $\{ \asma \}$, 
we consider this $<$-attack to be reverse and involving $\asma$: 
$\asmS' \cup \{ \asma \} \vdash^R \contr{\asmb_1}$, 
$\asmS' \subseteq \asmS$, $\asmb_1 \in \asmB$, 
and there is 
$\leqslant$-minimal 
$\asms' \in \asmS' \cup \{ \asma \}$ with $\asms' < \asmb_1$. 
% We can take $\asms'$ to be $\leqslant$-minimal such. 
Then, by WCP, there is $\asmS_1 \vdash^{R'_1} \contr{\asms'}$ with 
$\asmS_1 \subseteq ((\asmS' \cup \{ \asma \}) \setminus \{ \asms' \}) \cup \{ \asmb_1 \}$. 
Due to $<$-conflict-freeness of $\asmS \cup \{ \asma \}$, 
it holds that $\asmb_1 \in \asmS_1$. 
Also, due to $\leqslant$-minimality of $\asms'$ and because $\asms' < \asmb_1$,  
we find that $\nexists \asmx \in \asmS_1$ with $\asmx < \asms'$. 
Thus, $\asmS_1 \pattacks \{ \asms' \}$ via normal attack. 
Since $\asmS$ is $<$-admissible and $<$-defends $\{ \asma \}$, 
we must have $\asmS \pattacks \asmS_1$, 
and hence $\asmS \cup \{ \asma \} \pattacks \asmS_1$. 
This $<$-attack cannot be normal on $(\asmS' \cup \{ \asma \}) \setminus \{ \asms' \}$, 
due to $<$-conflict-freeness of $\asmS \cup \{ \asma \}$; 
while, if it is normal on $\asmb_1$, then $\asmS \cup \{ \asma \} \pattacks \asmB$, as required. 
Else, $\asmS \cup \{ \asma \} \pattacks \asmS_1$ via reverse attack: 
there is $\asmB_1 \vdash^{R_1} \contr{\asms_1}$ with 
$\asms_1 \in \asmS \cup \{ \asma \}$, $\asmB_1 \subseteq \asmS_1$, 
and $\exists \asms'_1 \in \asmB_1$ with $\asms'_1 < \asms_1$
% 
%. 
($\leqslant$-minimal). 
Due to $<$-conflict-freeness of $\asmS \cup \{ \asma \}$, we find $\asmb_1 \in \asmB_1$. 
Then again, by WCP, we find 
$\asmS_2 \vdash^{R'_2} \contr{\asms'_1}$, 
$\asmS_2 \subseteq (\asmB_1 \setminus \{ \asms'_1 \}) \cup \{ \asms_1 \}$, and $\asmb_1 \in \asmS_2$. 
Like with the proof of $<$-conflict-freeness, 
this process must terminate with a normal attack $\asmS \cup \{ \asma \} \pattacks \asmB$, 
so that $\asmS \cup \{ \asma \}$ eventually $<$-defends itself. 

Finally, we need to show that $\asmS \cup \{ \asma \}$ $<$-defends $\{ \asma' \}$. 
Given that $\asmS$ $<$-defends $\{ \asma' \}$ to begin with, 
and that $\pattacks$ is monotonic (Lemma \ref{lemma:attacks on supersets}), 
we conclude that $\asmS \cup \{ \asma \}$ $<$-defends $\{ \asma' \}$ too.
\end{proof}

\begin{proof}[Theorem \ref{theorem:flat ABA+ properties}]
Proof of each claim follows the pattern of the corresponding proofs in e.g.~\cite{Dung:1995,Bondarenko:Dung:Kowalski:Toni:1997,Dung:Mancarella:Toni:2007,Baroni:Cerutti:Giacomin:Guida:2011}.

\begin{enumerate}[label=(\roman*)]
\item In flat \abap\ frameworks all sets of assumptions are closed, 
and, in particular, $\emptyset$ is closed. 
Hence, $\emptyset$ is $<$-admissible, 
and so according to Theorem~\ref{theorem:ABA+ properties}(vii), 
$\F$ has a $<$-preferred extension.\footnote{Note that Weak Contraposition is not needed here, only flatness is required.}

\item Let $\asmE$ be a $<$-preferred extension of $\F$ and suppose for a contradiction that it is not $<$-complete. 
Let $\asmE$ $<$-defend some $\{ \asma \} \subseteq \A \setminus \asmE$. 
As $\asmE$ is $<$-admissible, $\asmE \cup \{ \asma \}$ is $<$-admissible, by Lemma~\ref{lemma:Fundamental}. 
But then $\asmE$ is not $\subseteq$-maximally $<$-admissible, 
contrary to $\asmE$ being $<$-preferred. 
Hence, by contradiction, $\asmE$ must be $<$-complete.

\item Follows from (i) and (ii) above. 

\item Define the $<$-defence operator $\Def: \wp(\A) \to \wp(\A)$ as follows: 
for $\asmA \subseteq \A$, 
$\Def(\asmA) = \{ \asma \in \A~:~\asmA \text{ $<$-defends } \{ \asma \} \}$. 
By Lemma \ref{lemma:attacks on supersets}, 
$\Def$ is monotonic: if $\asmA \subseteq \asmB \subseteq \A$, 
then $\Def(\asmA) \subseteq \Def(\asmB)$. 
As $(\wp(\A), \subseteq)$ is a complete lattice, 
fixed points of $\Def$ also form a complete lattice, 
according to Knaster-Tarski Theorem \cite{Tarski:1955}. 
As $\Def$ is compact (as deductions are finite), 
it has a unique least fixed point $G$, 
given by $G = \bigcup_{i \in \mathbb{N}} \Def^{~i}(\emptyset)$. 
As $\emptyset$ is $<$-admissible, $G$ is also $<$-admissible, 
by Lemma \ref{lemma:Fundamental}. 
Hence, $G$ is $<$-complete (as $G = \Def(G)$), 
and unique $\subseteq$-minimal such (as the least fixed point). 
Therefore, $G$ is a unique $<$-grounded extension of $\F$, 
and is $<$-complete, as required. 

\item From (i) above, we know that $\F$ admits $<$-preferred extensions, 
so let $\asmS$ be their intersection. 
If $\asmS = \emptyset$, then it is $<$-admissible, and so an $<$-ideal extension (unique). 
If $\asmS \neq \emptyset$ is $<$-admissible, then it is an $<$-ideal extension (unique as well). 
Else, assume $\asmS \neq \emptyset$ is not $<$-admissible. 
Then its $\subseteq$-maximally $<$-admissible subsets $\asmI \subsetneq \asmS$ are $<$-ideal extensions of $\F$. 
Suppose $\asmI$ and $\asmI'$ are two distinct $<$-admissible subsets of $\asmS$. 
Then their union $\asmI \cup I'$ is a subset of $\asmS$ too, and so $<$-conflict-free. 
By Lemma~\ref{lemma:Fundamental}, 
$\asmI \cup \asmI'$ $<$-defends itself, so must be $<$-admissible. 
Consequently, there can be only one $\subseteq$-maximally $<$-admissible subset of $\asmS$, 
i.e.~a unique $<$-ideal extension $\asmI$ of $\F$. 

Now, suppose for a contradiction that $\asmI$ is not $<$-complete. 
Then some $\{ \asma \} \subseteq \A \setminus \asmI$ is $<$-defended by $\asmI$. 
Such $\asma$ must be contained in the intersection $\asmS$ of $<$-preferred extensions of $\F$, 
because $\asmI \subseteq \asmS$ $<$-defends $\{ \asma \}$ 
and every $<$-preferred extension $\F$ is $<$-complete, by (ii) above. 
But then, $\asmI \cup \{ \asma \}$ is $<$-admissible, 
according to Lemma~\ref{lemma:Fundamental}, 
so that $\asmI$ is not $<$-ideal---a contradiction. 
Therefore, $\asmI$ must be $<$-complete.
\end{enumerate}
\vspace{-0.6cm}
\end{proof}

\begin{proof}[Theorem \ref{theorem:PAF to ABA+}]
Suppose first that $\asmE \subseteq \Args$ is a $\sigma$ extension of $\PAF$. 
As $\asmE$ is conflict-free in \DAF, it is $<$-conflict-free in \abafp, by Lemma \ref{lemma:PAF defeat iff <-attack}. 
Similarly, using Lemma \ref{lemma:PAF defeat iff <-attack} and construction of \abafp, it is plain to see that $\asmE$ $<$-defends itself. 
So $\asmE$ is $<$-admissible. 
It now suffices to prove additional properties as required for each semantics $\sigma$. 
We do this case by case. 

\begin{description}
\item[$\sigma$ = complete.]
If $\asmE$ $<$-defends $\asmA \subseteq \A$, then, by construction of \abafp, $\asmE$ $<$-defends every $\argA \in \asmA$, whence $\asmE$ defends every $\argA \in \asmA$ in \DAF. 
As $\asmE$ is complete, we find $\asmA \subseteq \asmE$, whence $\asmE$ is $<$-complete. 

\item[$\sigma$ = preferred.]
If $\asmE$ were not $\subseteq$-maximally $<$-admissible, 
it would $<$-defend some $\argA \in \A \setminus \asmE$ (as for $\sigma = $ complete above), 
whence $\asmE$ would defend $\argA$ in \DAF, 
and would not be preferred. 
Hence, $\asmE$ must be $<$-preferred. 

\item[$\sigma$ = stable.]
If $\asmE \npattacks \{ \argB \}$ for some $\argB \in \A \setminus \asmE$, 
then $\asmE \ndefeats \argB$, by Lemma \ref{lemma:PAF defeat iff <-attack}, 
so that $\asmE$ would not be stable, if $\asmE$ were not $<$-stable. 

\item[$\sigma$ = ideal.]
As $\asmE$ is contained in every preferred extension $P$ of $\PAF$, 
and since preferred extensions of $\PAF$ are in one-to-one correspondence with the $<$-preferred extensions of $\abafp$, 
as per proofs for $\sigma$ = preferred above and below, 
we conclude that $\asmE$ is contained in every $<$-preferred extension $P$ of $\abafp$. 
Clearly, $\asmE$ must be $\subseteq$-maximal such, 
as otherwise $\asmE$ would not be ideal in $\PAF$. 

\item[$\sigma$ = grounded.] 
$\asmE$ is $\subseteq$-minimally complete in \DAF\ \cite{Dung:1995}, and so in \PAF. 
Thus, $\asmE$ is $<$-complete as above, 
and due to $\subseteq$-minimality, it is the intersection of all $<$-complete extensions of \abafp, so $<$-grounded. 
\end{description}

Suppose now that $\asmE \subseteq \A$ is a $<$-$\sigma$ extension of $\abafp$. 
That $\asmE$ is admissible in \PAF\ follows from the construction of \abafp\ and Lemma \ref{lemma:PAF defeat iff <-attack}. 
For any $\sigma \in \{$\flatsemantics$\}$, a mirror argument as for $\sigma$ above applies here, 
in the case of $\sigma =$ grounded noting that, by construction of \abafp, the $<$-grounded extension of \abafp\ exists, is unique and $<$-complete. 
\end{proof}

%                                APPENDIX: ASPIC+
\section*{Appendix B. \aspicp}
\label{appendix:ASPIC+}

For our purposes in this paper, a simplified exposition of \aspicp, as follows, will suffice. 
(The reader is referred to \cite{Modgil:Prakken:2014,Prakken:2010,Modgil:Prakken:2013} for details.)

%       definition:aspicf
%\begin{definition}
\label{definition:aspicf}
An \emph{\aspicp\ framework} is a tuple \aspicf, where:
\begin{itemize}
\item $\LL$ is a language;
\item $\ccontrary: \LL \to \wp(\LL)$ is a \emph{contrariness} function such that: 
\begin{itemize}
\item $\varphi$ is a \emph{contrary} of $\psi$ just in case 
$\varphi \in \ccontr{\psi}$ and $\psi \not\in \ccontr{\varphi}$; 
\item $\varphi$ is a \emph{contradictory} of $\psi$, denoted $\varphi = -\psi$, just in case 
$\varphi \in \ccontr{\psi}$ and $\psi \in \ccontr{\varphi}$;  
\item every $\varphi \in \LL$ has at least one contradictory;
\end{itemize}
\item $\R_s$ is a set of \emph{strict} rules of the form $\varphi_1, \ldots, \varphi_n \to \varphi$, 
where $\varphi_1, \ldots, \varphi_n, \varphi \in \LL$; 
\item $\R_d$ is a set of \emph{defeasible} rules of the form $\varphi_1, \ldots, \varphi_n \To \varphi$, 
where $\varphi_1, \ldots, \varphi_n, \varphi \in \LL$; 
\item $\R_s \cap \R_d = \emptyset$; 
\item $\leqslant_d$ is a transitive binary relation on $\R_d$; 
\item $n: \R_d \to \LL$ is a \emph{naming} function for defeasible rules; 
\item $\K_n \subseteq \LL$ is a set of \emph{axioms};
\item $\K_p \subseteq \LL$ is a set of \emph{premises}; 
\item $\K_n \cap \K_p = \emptyset$; 
\item $\leqslant_p$ is a transitive binary relation on $\K_p$.
\end{itemize}
%\end{definition}

Whenever a component of an \aspicp\ framework is empty, we may omit it. 
In particular, as we will use neither the naming function $n$ for defeasible rules nor the axioms $\K_n$, these components are henceforth omitted. 
In the remainder of this section, we assume as given a fixed but otherwise arbitrary \aspicp\ framework \aspicfo. 

Arguments in \aspicp\ are defined as follows. 
%       definition:ASPIC+ argument
%\begin{definition}
\label{definition:ASPIC+ argument}
An argument $\argA$ is any of the following: 
\begin{itemize}
\item $[\varphi]$ iff $\varphi \in \K_p$. 
It has:
\begin{itemize}
\item \emph{premises} $\prem(\argA) = \{ \varphi \}$; 
\item \emph{conclusion} $\conc(\argA) = \varphi$; 
\item \emph{sub-arguments} $\sub(\argA) = \{ \argA \}$;
\item \emph{defeasible rules} $\defrules(\argA) = \emptyset$; 
\end{itemize}
\item $[\argA_1, \ldots, \argA_n \to \psi]$ iff $\argA_1, \ldots, \argA_n$ are arguments such that there is a strict rule 
$\conc(\argA_1), \ldots, \conc(\argA_n) \to \psi$ in $\R_s$. 
It has: 
\begin{itemize}
\item premises $\prem(\argA) = \prem(\argA_1) \cup \ldots \cup \prem(\argA_n)$; 
\item conclusion $\conc(\argA) = \psi$; 
\item sub-arguments $\sub(\argA) = \sub(\argA_1) \cup \ldots \cup \sub(\argA_n) \cup \{ \argA \}$;
\item defeasible rules $\defrules(\argA) = \defrules(\argA_1) \cup \ldots \cup \defrules(\argA_n)$; 
\end{itemize}
\item $[\argA_1, \ldots, \argA_n \To \psi]$ iff $\argA_1, \ldots, \argA_n$ are arguments such that there is a defeasible rule 
$\conc(\argA_1), \ldots, \conc(\argA_n) \To \psi$ in $\R_d$. 
It has: 
\begin{itemize}
\item premises $\prem(\argA) = \prem(\argA_1) \cup \ldots \cup \prem(\argA_n)$; 
\item conclusion $\conc(\argA) = \psi$; 
\item sub-arguments $\sub(\argA) = \sub(\argA_1) \cup \ldots \cup \sub(\argA_n) \cup \{ \argA \}$;
\item defeasible rules $\defrules(\argA) = \defrules(\argA_1) \cup \ldots \cup \defrules(\argA_n) \cup \{ \conc(\argA_1), \ldots, \conc(\argA_n) \To \psi \}$. 
\end{itemize}
\end{itemize}

The set of all arguments in \aspicfo\ is henceforth denoted by $\Args$. 
%\end{definition}

\aspicp\ frameworks are usually assumed to be \emph{closed under contraposition}, 
meaning that for all $\asmS \subseteq \LL$, $s \in \asmS$ and $\varphi \in \LL$ it holds that 
if there is an argument $\argA$ with $\conc(\argA) = \varphi$, $\defrules(\argA) = \emptyset$ and $\prem(\argA) \subseteq \asmS$, 
then there must be an argument $\argB$ with $\conc(\argB) = -s$, $\defrules(\argB) = \emptyset$ and $\prem(\argB) \subseteq (\asmS \setminus \{ s \}) \cup \{ -\varphi \}$. 
For our purposes, when mapping \abap~frameworks into \aspicp\ (see section \ref{subsec:ASPIC+})  
this is equivalent to assuming that \abafp\ satisfies \contraposition\ (Axiom \ref{axiom:WCP} in section \ref{sec:WCP}). 
Other requirements on \aspicp\ frameworks will be met automatically in our setting, 
so we omit to specify them here. 

Attacks in \aspicp\ are defined as follows.\footnote{When mapping \abap\ frameworks into \aspicp, due to absence of defeasible rules, only \emph{undermining} attacks will result, so we omit to specify the other types of attacks.}
%       definition:ASPIC+ attack
%\begin{definition}
\label{definition:ASPIC+ attack}
Let $\argA, \argB \in \Args$.
We say that $\argA$ \emph{attacks} $\argB$ (on $\argB'$), written $\argA \attacks \argB$, 
iff  
for some $\argB' \in \sub(\argB)$ such that $\argB' = [ \varphi ]$ and $\varphi \in \K_p$, 
$\conc(\argA) = -\varphi$ 
(i.e.~$\conc(\argA) \in \ccontr{\varphi}$ and $\varphi \in \ccontr{\conc(\argA)}$). 
%\end{definition}

Arguments can be compared using the given preference ordering $\leqslant_p$ over premises $\K_p$, as follows. 
%       definition:ASPIC+ comparison
%\begin{definition}
\label{definition:ASPIC+ comparison}
Consider $\argA, \argB \in \Args$ with  
$\asmA = \prem(\argA) \cap \K_p$ and $\asmB = \prem(\argB) \cap \K_p$. 
\begin{itemize}
\item $\argA \argeli \argB$ if $\exists \asma \in \asmA$ such that $\forall \asmb \in \asmB$ we find $\asma <_p \asmb$; 
\item $\argA \argdeli \argB$ if $\exists \asma \in \asmA \setminus \asmB$ such that $\forall \asmb \in \asmB \setminus \asmA$ we find $\asma <_p \asmb$; 
\item $\argA \argdemq \argB$ if $\forall \asma \in \asmA$ we find $\asmb \in \asmB$ with $\asma \leqslant_p \asmb$; 
\item $\argA \argdem \argB$ if $\argA \argdemq \argB$ and $\argB \argndemq \argA$. 
\end{itemize}

The three comparison principles \elitist, \democratic\ and \delitist\ are referred to as \emph{Elitist}, \emph{Democratic} and \emph{Disjoint Elitist} \cite{Young:Modgil:Rodrigues:2016}, respectively. 
%\end{definition}

\aspicp\ attacks succeed as defeats whenever the attacking argument is not less preferred than the (sub-)argument attacked, as follows. 
%       definition:ASPIC+ defeat
%\begin{definition}
\label{definition:ASPIC+ defeat}
Let $\argA, \argB \in \Args$ and fix $\prec \,\in \{ \argeli, \argdeli, \argdem \}$.  
We say that $\argA$ \emph{defeats} $\argB$, written $\argA \defeats \argB$, iff  
$\argA \attacks \argB$ on $\argB'$ and $\argA \nprec \argB'$. 
%\end{definition}
We may write e.g.~$\defeats_{\elitist}$ to denote the defeat with respect to (in this case) the Elitist comparison $\argeli$ of arguments. 

Finally, \aspicp~semantics are defined via the AA frameworks of arguments and defeats among them. 
More precisely, given $\F = \aspicfo$, 
the \emph{corresponding AA framework} is $\DAF$, 
where $\defeats$ is the defeat relation generated by $\attacks$ and $\prec$. 
Extensions of $\F$ (under various semantics $\sigma$) are defined as $\sigma$ extensions of $\DAF$ (see section \ref{sec:Background}).

\newpage
%                                REFERENCES
\singlespacing
\small{
\bibliographystyle{plain}
\bibliography{../../../Readings/library}
}

\end{document}